\documentclass[nohyperref]{article}

\usepackage{microtype}
\usepackage{graphicx}
\usepackage{booktabs} 

\usepackage{hyperref}

\usepackage[backend=biber,natbib=true,style=authoryear,doi=false,isbn=false,url=false,eprint=false,giveninits=true,maxbibnames=200,maxcitenames=2,mincitenames=1,uniquename=false,uniquelist=false,dashed=false]{biblatex}
    \usepackage[preprint,nonatbib]{neurips_2023}

\usepackage[utf8]{inputenc} 
\usepackage[T1]{fontenc}    
\usepackage{hyperref}       
\usepackage{url}            
\usepackage{booktabs}       
\usepackage{amsfonts}       
\usepackage{nicefrac}       
\usepackage{microtype}      
\usepackage{xcolor}         
\usepackage{enumitem}      
\usepackage{mdframed}
\usepackage{caption}
\usepackage{subcaption}  
\usepackage{amsmath}
\usepackage{amssymb}
\usepackage{mathtools}
\usepackage{amsthm}
\usepackage{empheq}
\definecolor{myblue}{HTML}{D2E4FC}
\newcommand*\mybluebox[1]{\fcolorbox{black}{myblue}{\hspace{1em}#1\hspace{1em}}}
\definecolor{linkcolor}{RGB}{82,82,192}
\hypersetup{
    colorlinks=true,
    linkcolor=linkcolor,
    citecolor=linkcolor
}

\newenvironment{myenv}[1]
  {\mdfsetup{
    frametitle={\colorbox{white}{\space#1\space}},
    innertopmargin=0pt,
    frametitleaboveskip=-\ht\strutbox,
    frametitlealignment=\center
    }
  \begin{mdframed}%
  }
  {\end{mdframed}}

\usepackage[capitalize,noabbrev]{cleveref}

\usepackage{thmtools}
\usepackage{thm-restate}

\theoremstyle{plain}
\newtheorem{theorem}{Theorem}[section]

\newtheorem{lemma}[theorem]{Lemma}

\theoremstyle{definition}
\newtheorem{definition}[theorem]{Definition}
\newtheorem{assumption}[theorem]{Assumption}

\theoremstyle{remark}
\newtheorem{remark}[theorem]{Remark}

\usepackage[textsize=tiny]{todonotes}

\usepackage{amsmath,amsfonts,bm}

\def\1{\bm{1}}

\DeclareMathAlphabet{\mathsfit}{\encodingdefault}{\sfdefault}{m}{sl}
\SetMathAlphabet{\mathsfit}{bold}{\encodingdefault}{\sfdefault}{bx}{n}

\def\sR{{\mathbb{R}}}

\newcommand{\E}{\mathbb{E}}

\DeclareMathOperator*{\argmin}{arg\,min}
\DeclareMathOperator{\Ima}{Im}

\def\Tiny{\fontsize{4pt}{4pt}\selectfont}
\newcommand*{\eqdef}{\ensuremath{\overset{\mathclap{\text{\Tiny def}}}{=}}}

\usepackage{tikz}
\usetikzlibrary{backgrounds}
\usetikzlibrary{arrows,shapes}
\usetikzlibrary{tikzmark}
\usetikzlibrary{calc}

\usepackage{mathtools, nccmath}
\usepackage{wrapfig}
\usepackage{comment}

\usepackage{array}
\usepackage{ragged2e}
\newcolumntype{P}[1]{>{\RaggedRight\hspace{0pt}}p{#1}}
\newcolumntype{X}[1]{>{\RaggedRight\hspace*{0pt}}p{#1}}

\usepackage{tcolorbox}

\usepackage{tikz}
\usetikzlibrary{arrows,shapes,positioning,shadows,trees,mindmap}
\usepackage[edges]{forest}
\usetikzlibrary{arrows.meta}
\colorlet{linecol}{black!75}
\usepackage{xkcdcolors} 

\usepackage{tikz}
\usetikzlibrary{backgrounds}
\usetikzlibrary{arrows,shapes}
\usetikzlibrary{tikzmark}
\usetikzlibrary{calc}

\title{Test like you Train in Implicit Deep Learning}

\author{%
 Zaccharie Ramzi\\
  DMA - ENS, CNRS\\
  \And
  Pierre Ablin\\
  Apple \\
  \And
  Gabriel Peyré\\
  DMA - ENS, CNRS \\
  \And
  Thomas Moreau\\
  Inria
}

\begin{document}

\maketitle

\begin{abstract}
Implicit deep learning has recently gained popularity with applications ranging from meta-learning to Deep Equilibrium Networks~(DEQs).
In its general formulation, it relies on expressing some components of deep learning pipelines implicitly, typically via a root equation called the inner problem.
In practice, the solution of the inner problem is approximated during training with an iterative procedure, usually with a fixed number of inner iterations.
During inference, the inner problem needs to be solved with new data.
A popular belief is that increasing the number of inner iterations compared to the one used during training yields better performance.
In this paper, we question such an assumption and provide a detailed theoretical analysis in a simple setting.
We demonstrate that overparametrization plays a key role: increasing the number of iterations at test time cannot improve performance for overparametrized networks.
We validate our theory on an array of implicit deep-learning problems.
DEQs, which are typically overparametrized, do not benefit from increasing the number of iterations at inference while meta-learning, which is typically not overparametrized, benefits from it.
\end{abstract}

\section{Introduction}
\label{introduction}

Implicit deep learning has seen a surge in recent years with various instances such as Deep Equilibrium Networks~(DEQs;~\citealt{baiDeepEquilibriumModels2019}), OptNets~\citep{amos2017optnet}, Neural ODEs~\citep{chen2018neural}, or meta-learning~\citep{pmlr-v70-finn17a,rajeswaran2019meta}.
In this work, we define implicit deep learning as the setting where intermediary outputs or task-adapted parameters $z^\star \in\sR^{d_z}$ are defined \emph{implicitly} as the solution of a root-finding problem involving the parameters $\theta\in\sR^{d_\theta}$ and the training set $\mathcal{D}_{\text{train}}$, that is, $f(z^\star, \theta, \mathcal{D}_{\text{train}})=0$.
Changing $\theta$ or $\mathcal{D}_{\text{train}}$ changes the landscape of $f$, thus $z^\star$ depends on $\theta$ and $\mathcal{D}_{\text{train}}$.
The optimal parameters $\theta$ are then found by minimizing a loss function $\ell(z^\star, \mathcal{D}_{\text{train}})$, which depends on $\theta$ implicitly through  $z^\star$.
In summary, the learning problem has the following \emph{bilevel} structure:
\begin{equation}
     \argmin_\theta \ell(z^\star(\theta, \mathcal{D}_\text{train}), \mathcal{D}_\text{train}) \quad
     \text{ s.t. } \quad f(z^\star(\theta, \mathcal{D}_\text{train}),\theta, \mathcal{D}_\text{train}) = 0.
    \label{eq:ideal-pb}
\end{equation}
We explain in detail how this formulation covers different cases such as Implicit Meta-learning~(iMAML;~\citealt{rajeswaran2019meta}) and Deep Equilibrium Models~(DEQs;~\citealt{baiDeepEquilibriumModels2019}) in \autoref{sec:related-works}.
The gradient of the loss $\ell$ relative to $\theta$ --called \emph{hypergradient}-- can be computed using the implicit function theorem~(IFT;~ \citealt{Krantz2013TheApplications, jaxopt_implicit_diff}), and then used to learn the optimal $\theta^\star$ by gradient descent.
Once the optimal $\theta^\star$ is learned, the model is used at inference with a new dataset $\mathcal{D}_\text{test}$, and we need to find a new root $z^\star(\theta^\star, \mathcal{D}_\text{test})$.

In the ideal formulation of \eqref{eq:ideal-pb}, the output of the model is independent of the root-finding procedure used to solve the inner problem.
In other words, \citet{baiequilibrium} explains for DEQs that there is "a decoupling between representational capacity and inference-time efficiency".

However, in most cases $z^\star(\theta)$ can only be approximated using an iterative procedure --note that for conciseness, we omit $\mathcal{D}_{\text{train}}$ in quantities that depend on it, but remain explicit for dependency on other datasets.
For example, DEQs rely on Broyden's method~\citep{broyden1965class} or Anderson Acceleration~\citep{anderson}, while meta-learning uses gradient-based methods.
Although some works have tried to tackle the question of how to speed up this iterative procedure by using a learned procedure~\citep{bai2022neural}, warm starting~\citep{micaelli2023recurrence,Bai_2022_CVPR}, or accelerating the hypergradient computation~\citep{fung2022jfb,ramzi2022shine}, in most cases, the number of iterations $N$ used for this procedure is fixed during training in order to keep a reasonable computational budget.\footnote{See the original implementations for \href{https://github.com/locuslab/deq}{DEQs} or \href{https://github.com/aravindr93/imaml_dev}{implicit meta-learning}}
We denote the resulting approximation $z_N(\theta)$.
The practical problem solved in implicit deep learning then becomes:
\begin{empheq}[box=\mybluebox]{equation*}\tag{P}\label{eq:pbp}
    \begin{aligned}
         \small
         \vphantom{\sum_0^i}\theta^{\star, N}  \in \argmin_\theta \ell(z_N(\theta))
         \text{ s.t. } z_N(\theta) \text{ is the $N$-th iterate of a procedure solving } f(z, \theta) = 0.
    \end{aligned}
\end{empheq}
Here, we highlight the dependency of the solution $\theta^{\star, N}$ on the number of inner iterations $N$.
At inference, one should decide how many iterations are used to solve the inner problem.
We ask:
\begin{myenv}{Question}
    \centering Is there a performance benefit in changing the number of inner iterations $N$ once the model is fitted for implicit deep learning?\

\vspace{-0.8em}\vphantom{1}
\end{myenv}
Because of the decoupling in the ideal case, many papers hypothesize that, when the model is fixed, a better approximation of the solution of the inner problem --i.e. increasing $N$-- could bring performance benefits~\citep{baiequilibrium,pal2022continuous}.
For instance, \citet{gilton2021deep} state that ``the computational budget can be selected at test time to optimize context-dependent trade-offs between accuracy and computation''.
However, they only show that performance deteriorates when the number of test-time inner iterations is lower than in training.
They also suggest that networks trained with IFT are more robust to changes in the number of inner iterations than their unrolled counterparts, where the hypergradient is computed via backpropagation through the inner solver steps.
\citet{anil2022path} ask whether more test-time iterations can help tackle "harder" problems -- for example increasing the dimension in a maze resolution problem.
They empirically show that if a DEQ has a property termed path-independence, it can benefit from more test-time iterations to improve the performance out-of-distribution.
However, there is little evidence that increasing the number of iterations at test-time can improve the performance for DEQs on the same data distribution.
In the meta-learning setup, the number of inner iterations can be greater during inference than during training, for example in the text-to-speech experiment of \citet{chen2020modular} --100 iterations during training vs 10 000 iterations at test-time.
\citet{pmlr-v70-finn17a} also showed that the performance increases when the number of inner iterations increases at test-time.

Formally, we would like to know if changing the number of iterations to approximate the root from $N$ to $N + \Delta N$ with $\Delta N > -N$ --i.e., using $z_{N + \Delta N}(\theta^{\star, N})$ instead of $z_N(\theta^{\star, N})$-- can yield better performance.
In this paper, we answer that question first with a theoretical analysis in a simple case and then with extensive empirical results.
In our theoretical analysis, we study ${D(N, \Delta N) \eqdef \ell(z_{N + \Delta N}(\theta^{\star, N})) - \ell(z_{N}(\theta^{\star, N})) }$, the training loss increase when changing the number of inner iterations by $\Delta N$ for a fixed learned $\theta^{\star, N}$.
This quantity is a proxy for the increase in test loss, provided we have access to enough training data.
On the other hand, we consider changes in test-set metrics in our experiments.

Theoretically, we uncover for  overparametrized models a phenomenon we term \textbf{Inner Iterations Overfitting~(I2O)}, in which there is no benefit in increasing the number of inner iterations.
We empirically find that DEQs suffer from I2O, while we confirm the benefit of increasing the number of inner iterations for the meta-learning setup.
Our contributions are the following:
\begin{itemize}[leftmargin=*]
    \item The \textbf{theoretical demonstration of I2O} for the case where $f$ is affine. In \autoref{thm:quadratic}, we derive a lower bound on $D(N, \Delta N)$ and characterize two regimes in subsequent corollaries: overparametrization in $\theta$ which leads to I2O and no overparametrization. This provides a practical guideline for DEQs that fall close to the overparametrization regime: in order to achieve the best performance at test time, one should use the same number of inner iterations as in training.
    \item The \textbf{empirical demonstration of I2O} for DEQs on diverse tasks such as image classification, image segmentation, natural language modeling and optical flow estimation. This validates the practical guideline established above, and also the current practice.
    We also show that I2O is much less prevalent for meta-learning cases.
    \item The \textbf{comparison of robustness to changes in inner iterations number}  between IFT and unrolling. We show with \autoref{thm:iftcvg} that the choice of hypergradient computation does not impact whether implicit deep learning suffers from I2O or not. We also highlight this phenomenon empirically for DEQs.
\end{itemize}

\section{Background on Implicit Deep Learning}
\label{sec:related-works}

\vskip.5em\textbf{DEQs~}
DEQs were introduced by \citet{baiDeepEquilibriumModels2019} for NLP and have since then been used for a variety of tasks including computer vision~\citep{baimultiscalemodels2020,micaelli2023recurrence}, inverse problems~\citep{gilton2021deep,zou2023deep} or optical flow estimation~\citep{Bai_2022_CVPR}.
On optical flow estimation~\citep{Bai_2022_CVPR} and landmark detection~\citep{micaelli2023recurrence}, they have even managed to achieve a new state of the art.
In their ideal formulation, DEQs are trained by minimizing a task-specific loss on a dataset, where the output of the model is obtained by computing the fixed point of a nonlinear function~\citep{baiDeepEquilibriumModels2019}:
\begin{equation}
    \label{eq:deq-formulation}
    \argmin_\theta \sum_i \ell(z^\star_i(x_i, \theta), y_i) \quad \text{s.t.} \quad z^\star_i(x_i, \theta) = g(z^\star_i(x_i, \theta), x_i, \theta)
\end{equation}
They can be fitted in the framework introduced in Eq.~\eqref{eq:ideal-pb} by lifting the inner variables $z_i$ to a stacked version $\mathbf{z} = [z_1, \ldots, z_n]^\top$ and similarly for the inner functions and the outer losses.
The inner problem can also be cast as a root finding problem rather than a fixed point problem, by simply considering $f(z, x_i, \theta) = z - g(z, x_i, \theta)$.
The gradient of the outer losses with respect to the parameters $\theta$ is then computed using the IFT, which yields:
\begin{equation}
    \label{eq:hypergrad-ift}
    \nabla L_i =  -  \left( \partial_z f\left(z^\star, x_i, \theta \right)^{-1} \partial_\theta f\left(z^\star, x_i, \theta \right) \right)^\top \nabla_z \ell\left(z^\star, y_i\right),
\end{equation}
with $L_i(\theta) = \ell(z^\star_i(x_i, \theta), y_i)$.
Importantly, this gradient computation does not depend on the procedure used to compute $z^\star$.
Therefore, the intermediary outputs (activations) used to compute it do not need to be saved in order to compute the gradient, leading to a memory-efficient scheme.

\vskip.5em\textbf{(i)MAML~}
Model-agnostic Meta-learning~(MAML) was introduced by \citet{pmlr-v70-finn17a} as a technique to train neural networks that can quickly adapt to new tasks.
The idea is to learn a meta model, such that its parameters, when trained on a new task, yield good generalization performance.
Formally, MAML has the following formulation:
\begin{equation}
    \label{eq:meta-learning}
        \argmin_{\theta^\text{(meta)}} \sum_i \ell\left(\theta_i, \mathcal{X}_i^\text{(val)}\right)\quad \text{s.t.} \quad \theta_i = \theta^\text{(meta)} - \alpha \nabla_\theta \ell(\theta^\text{(meta)}, \mathcal{X}_i^\text{(train)}),
\end{equation}
where we omitted the dependence of $\theta_i$ on $\theta^\text{(meta)}$ and $\mathcal{X}_i^\text{(train)}$ for conciseness.
The right hand-side corresponds to the training on a new task: a single gradient descent step with step size $\alpha$ to minimize a task specific training loss, the task $i$ being defined by its training and validation datasets $\mathcal{X}_i^\text{(train)}$ and $\mathcal{X}_i^\text{(val)}$.
The task-adapted parameters $\theta_i$ are then used to compute the loss on the validation set to measure how well these parameters generalize.
In a follow-up work, \citet{rajeswaran2019meta} introduced the implicit MAML~(iMAML) formulation, where the gradient descent step is replaced by the minimization of a regularized task specific loss.
Formally, the problem is:
\begin{equation}
    \label{eq:imaml}\argmin_{\theta^\text{(meta)}} \sum_i \ell\left(\theta_i, \mathcal{X}_i^\text{(val)}\right) \quad
    \text{s.t.} \quad \theta_i \in \argmin_\theta \underbrace{\ell(\theta, \mathcal{X}_i^\text{(train)}) + \frac{\lambda}{2} \|\theta - \theta^\text{(meta)}\|_2^2}_{F(\theta, \theta^\text{(meta)})}
\end{equation}
This formulation can easily be cast into the form of Eq.~\eqref{eq:ideal-pb} by replacing the inner optimization problem with the associated root problem on the gradient $\nabla_\theta F(\theta, \theta^\text{(meta)})=0$.
Then, similarly to the DEQ setting, the task adapted parameters $\theta_i$ can be stacked together to form a single inner variable $\mathbf{z} = [\theta_1, \ldots, \theta_n]^\top$.


\section{Theory of affine implicit deep learning}
\label{sec:theory}
In order to study the I2O phenomenon, we will make some simplifying assumptions on the nature of the inner procedure of problem~\eqref{eq:pbp} as well as the functions studied.
First, we consider that the procedure to solve the inner problem is a fixed-point iteration method with fixed step size.
Second, we restrict ourselves to the case where the inner problem is affine.
Third, we consider only cases where the outer loss is quadratic.

Before we proceed to the formal demonstration of I2O, let us give some intuition.
When the inner problem is overparametrized in the outer variable, it means that we can tune the approximate output of the procedure $z_N$ to minimize exactly the outer loss.
But this tuning is highly dependent on the procedure, and therefore on the number of inner iterations $N$ used.
If it is changed while keeping the same learned outer variable, the outer loss cannot decrease because it is already minimized.

\paragraph{Main result and corollaries}
Let us formalize the assumptions used to prove our main result.

\begin{assumption}[Fixed-point iteration]
    \label{ass:fpi} The procedure to solve the inner problem $f(z, \theta) = 0$ is a fixed-point iteration method with fixed step size $\eta >0$ and initialization $z_0 \in \sR^{d_z}$.
    This means that:
    \begin{equation}
        \label{eq:fpi}
        z_{N+1}(\theta) = z_{N}(\theta) - \eta f(z_{N}(\theta), \theta)\enspace.
    \end{equation}
\end{assumption}
For iMAML, this corresponds to gradient descent to solve the task adaptation, with $f = \nabla_z F$ with $F$ the regularized task specific loss.
Note that it does not perfectly match the practice for DEQs which are usually trained with Broyden's method~\citep{broyden1965class} or Anderson acceleration~\citep{anderson}.
Also, since $z_0$ does not depend on $\theta$, this does not cover the MAML setting.

\begin{assumption}[Affine inner problem]
    \label{ass:aff-inner} $f: \sR^{d_z \times d_\theta} \to \sR^{d_z}$ is an affine function:
    \begin{equation}
        \label{eq:aff}
        f(z, \theta) = K_\text{in}^\top(Bz + U\theta + c),
    \end{equation}
    with $K_\text{in}, B \in \sR^{d_x \times d_z}, U \in \sR^{d_x \times d_\theta}, c \in \sR^{d_x}$, with $K_\text{in}$ surjective, i.e. $d_x \leq d_\theta$.
    Moreover, $BK_\text{in}^\top$ has eigenvalues with positive real part.
\end{assumption}
\autoref{ass:aff-inner} corresponds to considering a DEQ with an affine layer, a setting commonly used to study DEQs~\citep{kawaguchi2021theory}, although in practice $f$ is nonlinear.
This class of function corresponds to affine functions for which the fixed point iterations~\eqref{eq:fpi} converge (see \autoref{app:affine-inner}).
In particular, it is more general than the simple case of $K_\text{in} = I$ with $B$ whose eigenvalues have a positive real part.
For iMAML, this corresponds to meta-learning a linear model with a quadratic regression loss.
We show these two correspondences in \autoref{app:affine-inner}.
More generally, it includes cases where $f$ is the gradient of a convex lower-bounded quadratic function $F(z, \theta) = \frac12 \|K_\text{in}z + U\theta + c\|_2^2$, with $B = K_\text{in}$, a setting studied by \citet{pmlr-v162-vicol22a}.
We provide an extended review of this work in \autoref{app:related}.

\begin{assumption}[Quadratic outer loss]
    \label{ass:quadratic}
    $\ell$ is a convex quadratic function bounded from below:
    \begin{equation}
        \ell(z) = \frac12 \|K_\text{out} z - \omega\|_2^2,
    \end{equation}
    with $\omega \in \sR^{d_\omega}$ and $K_\text{out} \in \sR^{d_z \times d_\omega}$.
\end{assumption}

\autoref{ass:quadratic} is meaningful, typically in inverse problems~\citep{zou2023deep} or meta-learning regression~\citep{pmlr-v70-finn17a,rajeswaran2019meta} where the output of the inner problem, a recovered signal, is compared to a ground truth signal.
It does not include different types of problems such as image classification~\citep{baimultiscalemodels2020} that would require more complex losses such as cross-entropy.

Before stating our main result, we recall that our objective is to control ${D(N, \Delta N) \eqdef \ell(z_{N + \Delta N}(\theta^{\star, N})) - \ell(z_{N}(\theta^{\star, N}))}$, the loss increase when the number of inner iterations changes by $\Delta N$.

\begin{myenv}{Main result}
    \begin{restatable}[Inner Iterations Overfitting for affine inner problems]{thm}{quadratic}
        \label{thm:quadratic}
        Under \autoref{ass:fpi}, \autoref{ass:aff-inner} and \autoref{ass:quadratic}, we have:
        \begin{equation}
            \label{ineq:main-ineq}
            D(N, \Delta N) \geq - \frac12 \| \textcolor{xkcdBluishGreen}{\left(\mathcal{P}(K_\text{out} K_\text{in}^\top) - \mathcal{P}(K_\text{out} K_\text{in}^\top E_N U )\right)}  \textcolor{xkcdMauve}{(K_\text{out} r_N - \omega)}  \|_2^2,
        \end{equation}
        where $\mathcal{P}(X)$ is the orthogonal projection on $\mathrm{range}(X)$, ${E_N = \left((I - \eta B K_\text{in}^\top )^N - I\right) (B K_\text{in}^\top )^{-1}}$ and $r_N = K_\text{in}^\top E_N  (Bz_0 + c) +  z_0$.
    \end{restatable}
\end{myenv}

The detailed proof is in \autoref{app:proof-main} where we also cover bilevel optimization settings.
We give closed-form expressions for $\ell(z_{N + \Delta N}(\theta^{\star, N}))$ and $D(N, \Delta N)$.

The lower bound of $D(N, \Delta N)$ is independent of $\Delta N$ which means that for every $N$ the maximum decrease in loss achievable by increasing the number of inner iterations at test time is bounded from below.
This lower bound is the negative squared norm of a vector ${\textcolor{xkcdBluishGreen}{\left(\mathcal{P}(K_\text{out} K_\text{in}^\top) - \mathcal{P}(K_\text{out} K_\text{in}^\top E_N U )\right)} \textcolor{xkcdMauve}{(K_\text{out} r_N - \omega)} }$.
Let us describe each component of this expression.
$\textcolor{xkcdMauve}{(K_\text{out} r_N - \omega)} $ is a vector that encompasses elements from the inner problem and procedure ($r_N$) and from the outer problem ($\omega$, $K_\text{out}$).
$\textcolor{xkcdBluishGreen}{\left(\mathcal{P}(K_\text{out} K_\text{in}^\top) - \mathcal{P}(K_\text{out} K_\text{in}^\top E_N U )\right)}$ is a difference of projectors which measures how much $U$ is "not surjective relative to $K_\text{out} K_\text{in}^\top E_N$", and this is highlighted by the next two corollaries which explain in more details the role the surjectivity --i.e. the column rank-- of $U$ in this lower bound.

\begin{restatable}[Overparametrization in $\theta$]{cor}{overparam}
    \label{cor:overparam}
    Under \autoref{ass:fpi}, \autoref{ass:aff-inner} and \autoref{ass:quadratic}, if $U$ is surjective, we have for all $\Delta N$:
    \begin{equation}
        \label{ineq:main-ineq-overparam}
        \ell(z_{N + \Delta N}(\theta^{\star, N})) \geq \ell(z_N(\theta^{\star, N})),
    \end{equation}
\end{restatable}

\begin{proof}
    When $U$ is surjective, $\mathcal{P}(K_\text{out} K_\text{in}^\top E_N U ) = \mathcal{P}(K_\text{out} K_\text{in}^\top E_N )$.
    Further, $\forall N \geq N_0$, $E_N$ is invertible.
    Therefore, $\mathcal{P}(K_\text{out} K_\text{in}^\top E_N ) = \mathcal{P}(K_\text{out} K_\text{in}^\top)$.
    We can conclude using \autoref{thm:quadratic}.
\end{proof}

In other words, when the inner model is sufficiently expressive, a regime called overparametrization, the inner variable is simply reparametrized with the outer variable, allowing the overall procedure to find the global minimum.

We numerically validate \autoref{cor:overparam} with small-scale experiments on a quadratic bilevel optimization problem, and show the results in \autoref{fig:quad-U-surj} (left).

\begin{figure}
    \begin{subfigure}[t]{0.39\textwidth}
        \centering
        \includegraphics{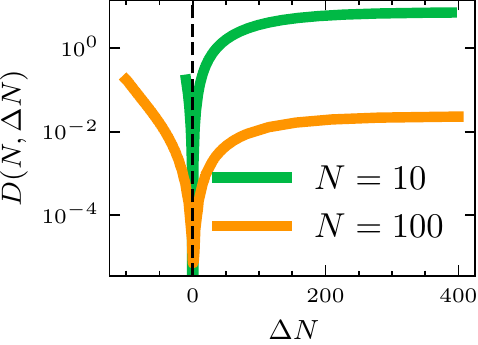}
    \end{subfigure}
    \hfill
    \begin{subfigure}[t]{0.6\textwidth}
        \centering
        \includegraphics{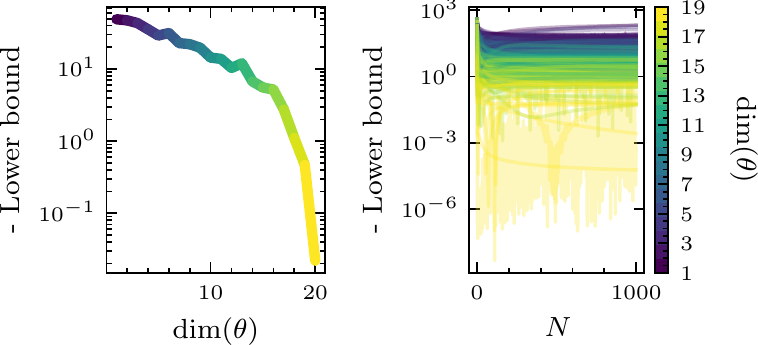}
    \end{subfigure}
    \caption{\label{fig:quad-U-surj} \textbf{Illustrations of the corollaries}. The left-hand side figure illustrates \autoref{cor:overparam}, i.e., the overparametrization regime.
    The black dashed line is at $\Delta N = 0$.
    We consider $z$ of dimension 5, $\theta$ of dimension 4, $F$ whose gradient is $f$ and $\ell$ convex but not strongly convex quadratic functions.
    The right-hand side figure illustrates \autoref{cor:avgcase}, i.e. the average case.
    We report the Negative Lower bound, i.e. $\frac12 \| \textcolor{xkcdBluishGreen}{\left(\mathcal{P}(K_\text{out} K_\text{in}^\top) - \mathcal{P}(K_\text{out} K_\text{in}^\top E_N U )\right)}  \textcolor{xkcdMauve}{(K_\text{out} r_N - \omega)}  \|_2^2$, from \autoref{thm:quadratic}. The inner and outer problem are strongly convex, and the dimension of $z$ is 20. The inner problem would be overparametrized in $\theta$ on average for dimension 20. We compute the lower bound for different inner optimization times and 20 seeds. In the left panel, we show the negative lower bound for $N=100$ averaged over all the seeds for each dimension of $\theta$. In the right panel, we show the negative lower bounds for all seeds and for all $N$ for different dimensions of $\theta$.}
\end{figure}

However, one can wonder what can be said when we are not in the overparametrized regime.
The next corollary shows what happens when the matrix $U$ is not overparametrized, and its entries are drawn i.i.d. from a Gaussian distribution.
\begin{restatable}[Average case]{cor}{avgcase}
    \label{cor:avgcase}
    Under \autoref{ass:fpi}, \autoref{ass:aff-inner} and \autoref{ass:quadratic}, if $K_\text{in}$ is invertible and $\ell$ is strongly convex, we have:
    \begin{equation}
        \label{ineq:main-ineq-avgcase}
        \E_{U \sim \mathcal{N}(0, I)} \left[D(N, \Delta N) \right] \geq - \frac12 \textcolor{xkcdBluishGreen}{(1 - \frac{\min(d_x, d_\theta)}{d_x})} \textcolor{xkcdMauve}{ (\rho(K_\text{out}) \|r_\text{max}\|_2^2 + \|\omega\|_2^2)},
    \end{equation}
    with $\rho(K_\text{out})$ the spectral radius of $K_\text{out}$ and $\|r_\text{max}\|_2^2 \in \max_N \|r_N\|_2^2$.
\end{restatable}

\begin{proof}
    The proof relies on the computation of the expected value of the norm of the projection on the image of a matrix for a matrix with random coefficients.
    It is given in full in \autoref{app:proof-main}
\end{proof}
An edge case of \autoref{cor:avgcase} is the situation where $d_x \leq d_\theta$, because the left term cancels.
But this situation is on average equivalent to the overparametrization regime because we go from a $d_\theta$-dimensional space to a $d_x$-dimensional space.
What \autoref{cor:avgcase} tells us is that as we get closer to overparametrization by increasing the dimension of the outer variable $\theta$, the expected loss increase we could get by changing the number of inner iterations gets closer to 0.

We ran an experiment to illustrate \autoref{cor:avgcase} whose results are shown in \autoref{fig:quad-U-surj} (right).
As expected, the lower bound of \autoref{thm:quadratic} does not vary significantly for different numbers of inner iterations.
Moreover, we observe that the lower bound decreases in magnitude as the inner problem is more and more overparametrized in $\theta$.
In cases where the inner or outer problems are not strongly convex, the lower bound can be 0 even before $U$ is surjective.
We show such cases in \autoref{app:add-results}.

In order to understand in which regime fall DEQs and meta-learning we need to understand to which extent the inner problem is close to overparametrization in the outer variable $\theta$.
If one wants to be close to overparametrization on average, this requires an outer variable $\theta$ of dimension $d_z \times n$, where $n$ is the number of samples.
While this number is usually prohibitively large, it is a common assumption in deep learning setups to assume that one can overfit the training data~\citep{li2018learning,du2018gradient,arora2019exact}.
However, in the case of meta-learning, since we are learning on multiple tasks at the same time, it is impossible a priori to overfit all the tasks at the same time since they might have contradicting objectives.
Therefore, we expect DEQs to have a hard time benefiting from more iterations, while there is room for meta-learning to do so.

We stress that these results do not cover the generalization to a test dataset $\mathcal{D}_\text{test}$.
Indeed, the quantity $D(N, \Delta N)$ only monitors the increase in training loss achieved by changing the number of inner iterations.
However, $D(N, \Delta N)$ is a good proxy for the increase in test loss provided a large enough training data set.

\paragraph{Implicit differentiation for affine inner problems}
It can be noted that \autoref{thm:quadratic} considers $\theta^{\star, N}$ as one of the solutions to \eqref{eq:pbp}.
However, most of the time in practice~\citep{baiDeepEquilibriumModels2019,rajeswaran2019meta}, the optimization is actually performed using the approximate implicit differentiation gradients, rather than the true unrolled gradients in order to have a memory-efficient training.
For an inner function $f$, this descent, with constant step size $\alpha_N$ (i.e. independent of $T$), would typically be written as the following:

\begin{equation}
    \label{eq:ift-practical-descent}
    \begin{split}
        &\theta^{T + 1 , N}_\text{IFT} = \theta^{T, N}_\text{IFT} - \alpha_N  p_N(\theta^{T, N}_\text{IFT})\\
        &\text{where } p_N(\theta) = -(\partial_z f(z_N, \theta)^\dagger \partial_{\theta} f(z_N, \theta))^\top \nabla \ell(z_N),
    \end{split}
\end{equation}
where $\dagger$ stands for the pseudo-inverse.
This equation is a practical implementation of \eqref{eq:hypergrad-ift} in the case where the Jacobian is not necessarily invertible, and we use an approximate inner solution $z_N$.
This formula is heuristic and does not necessarily provide a descent direction for arbitrary $N$.
We defer the proofs of the following results to \autoref{app:proof-ift}.

\begin{restatable}[Convergence of the practical IFT gradient descent]{lmm}{iftcvg}
    Under \autoref{ass:fpi}, \autoref{ass:aff-inner} and \autoref{ass:quadratic}, with $\ell$ strongly convex, $\exists N_0$ such that $\forall N > N_0$,  $\exists \alpha_N > 0$ such that the sequences $\theta^{T, N}_\text{IFT}$ converge to a value denoted $\theta^{\star, N}_\text{IFT}$, dependent only on the initialization.
\end{restatable}

Therefore, practical optimal solutions, denoted $\theta^{\star, N}_{\text{IFT}}$, are solutions to the following root problem:
\begin{equation}
    \label{eq:ift-practical}
        (\partial_z f(z_N(\theta), \theta)^\dagger \partial_{\theta} f(z_N(\theta), \theta))^\top \nabla \ell(z_N(\theta)) = 0
\end{equation}

\begin{restatable}[Equivalence of IFT and unrolled solutions for affine inner problems]{thm}{iftequivunrolled}
    \label{thm:iftcvg}
    Under \autoref{ass:fpi}, \autoref{ass:aff-inner} and \autoref{ass:quadratic}, with $\ell$ strongly convex and $U$ surjective, we have:
    \begin{equation}
        \theta^{\star, N}_{\text{IFT}} = \theta^{\star, N}
    \end{equation}

\end{restatable}

For overparametrized cases, where $U$ is surjective, this means that \autoref{cor:overparam} is valid for IFT based implicit deep learning.
Therefore, this shows that Implicit Deep Learning trained with IFT is not less prone to I2O than if it is trained through unrolling in this case.
We confirm this empirically in \autoref{sec:stability} for practical cases with DEQs.


\section{The empirical phenomenon of inner iterations overfitting}
\label{sec:exp-isto}

In order to validate the results from our theoretical analysis in realistic cases, we explore the I2O phenomenon for DEQs and (i)MAML experiments, two settings that highlight the different regimes from our theoretical results.

\paragraph{DEQs}
We conduct experiments on pre-trained DEQs in various successful applications.
The evaluation is unchanged except for the number of inner iterations, which is the same for training and inference in the typical \texttt{deq}\footnote{\href{https://github.com/locuslab/deq}{github.com/locuslab/deq}} library~\citep{baiDeepEquilibriumModels2019}.
The settings we cover are text completion on Wikitext~\citep{baiDeepEquilibriumModels2019,merity2017pointer}, large-scale image classification on ImageNet~\citep{baimultiscalemodels2020,deng2009imagenet},  image segmentation on Cityscapes~\citep{baimultiscalemodels2020,Cordts2016Cityscapes} and optical flow estimation on Sintel~\citep{Bai_2022_CVPR,Butler:ECCV:2012}.
We report the test performance gap in \autoref{fig:deq-inner-opt-time-overf-all}, i.e. the difference between the test performance for $N+ \Delta N$ inner iterations and the test performance for $N$ inner iterations and give more details on the experiments in \autoref{app:exp_details}.
The test performance is always cast as lower is better.

\begin{figure}
    \centering
    \includegraphics{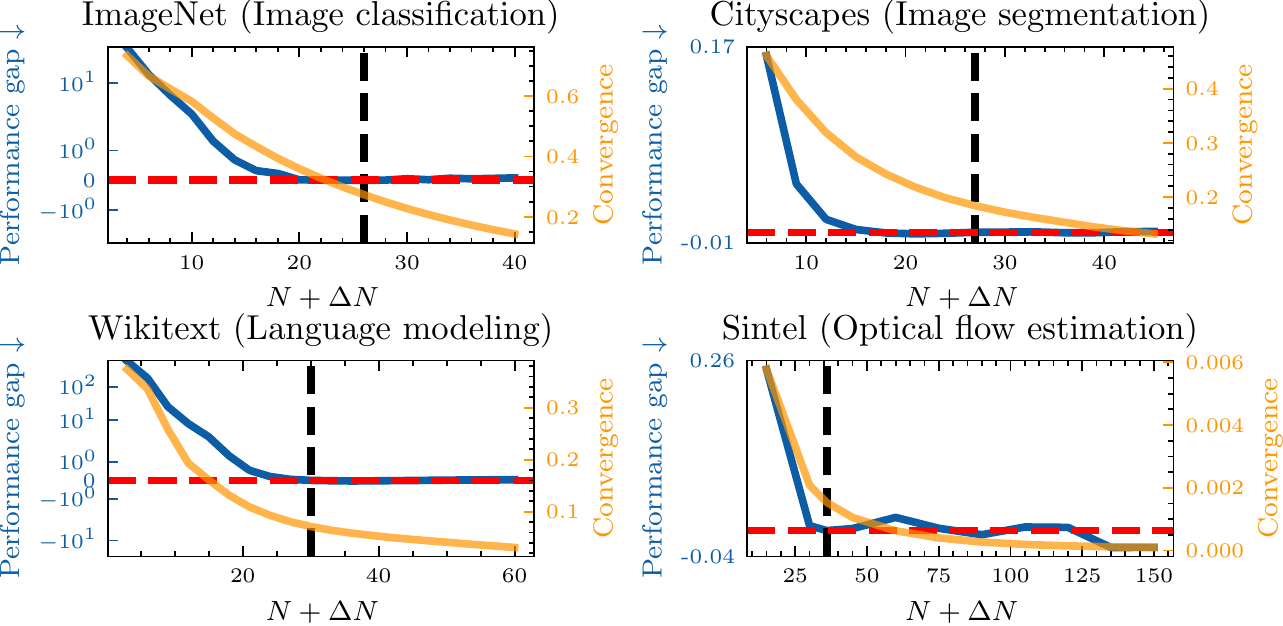}
    \caption{
    \textbf{I2O for DEQs}:
    Test performance gap (lower is better) using $N+\Delta N$ inner iterations at inference compared to using $N$ inner iterations. This performance gap reaches 0 after reaching $N$ the number of iterations used during training. The black dashed line is at $\Delta N = 0$, i.e. the training number of inner iterations. The red dashed line is at 0. For ImageNet, the performance is measured using the top-1 error rate (\%), for Cityscapes it is measured using the negative mean IoU, for WikiText it is measured using the perplexity and for optical flow it is measured using the average EPE.
    Note that after $N$ iterations, getting closer to convergence does not bring a performance benefit.
    }
    \label{fig:deq-inner-opt-time-overf-all}
\end{figure}

The figure shows that for all four cases, the performance of the model does not improve when increasing the number of iterations at test time (\textcolor{xkcdBlue}{blue} lines).
Indeed, once $\Delta N$ becomes positive, the performance tends to plateau.
For the optical flow estimation case, while the performance does get better for a very large number of iterations, there is no clear trend associated with it, since we also see it degrading for a small positive $\Delta N$.
We also observe that in all cases, using fewer iterations at inference is almost always detrimental.
This highlights I2O for DEQs, which are overparametrized (e.g. reaching 90\% training accuracy on ImageNet compared to 80\% test accuracy).

Similar observations can be made on the training performance and on training and test losses, as demonstrated in \autoref{fig:deq-inner-opt-time-overf-imagenet-loss} and \autoref{fig:deq-inner-opt-time-overf-imagenet-train} in appendix, on ImageNet.
This shows that while our theoretical results uncover this phenomenon in simple cases for the training loss only, it is observable for more complex setups, even for the test performance gap.

We verified that these observations were not artefacts of the convergence being already attained or stuck due to some instability by plotting it alongside the performance (\textcolor{xkcdOrange}{orange} lines).
The figures show that increasing the number of iterations above $N$ produces better approximation of the inner problem's solution, while not improving the performances.

Additional experiments for single image super resolution are also presented in \autoref{fig:deq-inner-opt-time-overf-ip} in appendix, with a classical DEQ architecture and the recently introduced ELDER method~\citep{zou2023deep}.
The results of these experiments show similar findings as the ones from \autoref{fig:deq-inner-opt-time-overf-all}.

\paragraph{(i)MAML}
Unlike DEQs, (i)MAML is less prone to I2O.
For example, \citet{chen2020modular} meta-train a network with $N=100$ inner steps, and meta-test it with $N+\Delta N = 10 000$ inner steps.
We run experiments on the synthetic sinusoids regression task introduced by \citet{pmlr-v70-finn17a} and confirm that I2O is much less prevalent in (i)MAML.
\autoref{fig:maml} shows the variation in test performance for the meta-learning task when changing the number of inner iterations.
More details on this experiment are given in \autoref{app:exp_details}.

\begin{figure}
    \centering
    \includegraphics{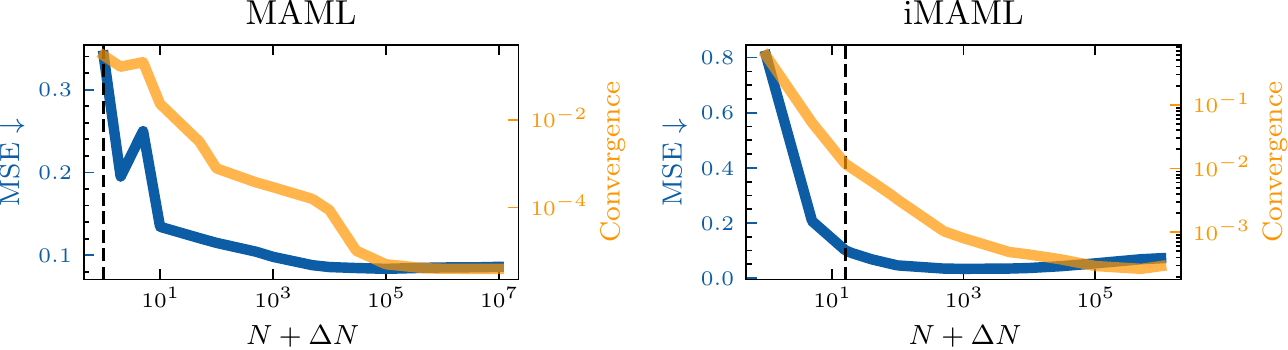}
    \caption{
    \textbf{(i)MAML is less prone to I2O}:
    Test MSE using $N+\Delta N$ inner iterations at inference. The black dashed line is at $\Delta N = 0$, i.e. the training number of inner iterations. Note the log-scale for the x-axis.
    Note that the best MSE is reached before convergence.
    }
    \label{fig:maml}
\end{figure}

We see that while the best performance at test time is achieved for a number of inner iterations much larger than the one used during training, this improvement is bounded from below as predicted by \autoref{thm:quadratic}: the best MSE is reached before convergence.
As for DEQs, reaching convergence does not provide the best performance.

Further, we highlight the effect of overparametrization in the iMAML case.
As the overparametrization is not easy to measure, we use overfitting as a proxy: the better the model is at overfitting, the more overparametrized it is.
\autoref{fig:maml-overfit} shows the training loss increase for various training set sizes for meta-learning, from 1 meta-batch --easy to overfit-- to 25 meta-batches --harder to overfit.
We observe that the better the model is at overfitting the training dataset --i.e. low $\ell(z_N(\theta^{\star, N}))$-- the less it benefits from more iterations at inference --i.e. higher $D(N, \Delta N)$.
Note that we conduct this experiment on the training loss, as the test data is too different from the training one with so few meta-batches.

\begin{figure}
    \centering
    \includegraphics{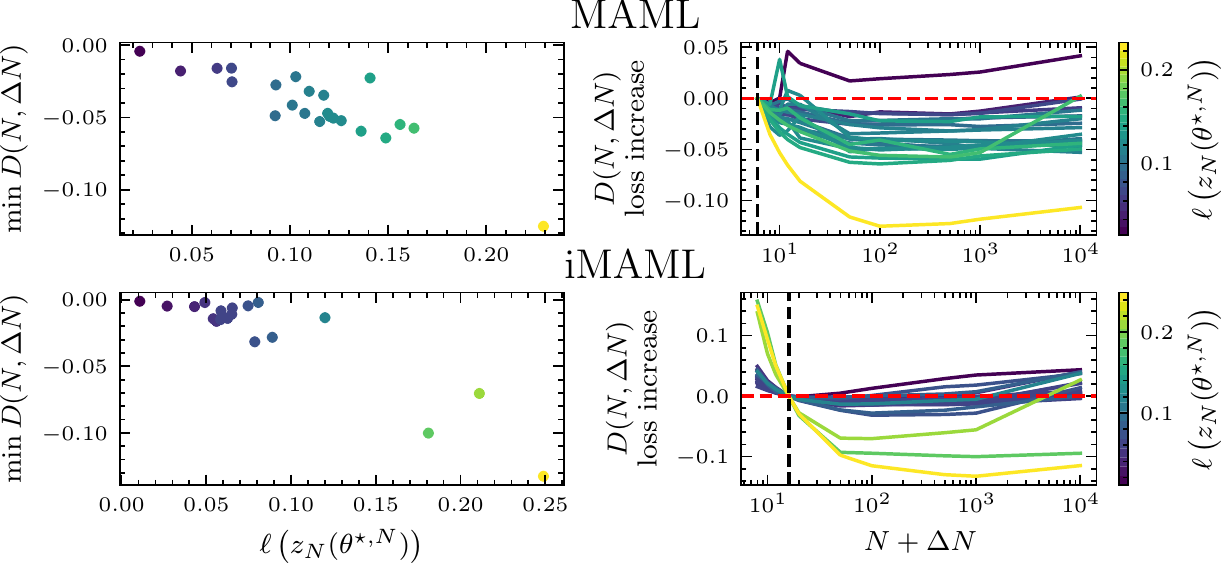}
    \caption{
    \textbf{The impact of overparametrization for (i)MAML}:
    Training loss increase $D(N, \Delta N)$ for different levels of overparametrization.
    This level of overparametrization is measured by the loss on the training set for $N$ inner iterations, $\ell(z_N(\theta^{\star, N}))$.
    Generally, the more meta-batches there are in the training set, the higher this loss.
    The black dashed line is at $\Delta N = 0$, i.e. the training number of inner iterations.The red dashed line is at $D = 0$, i.e. no increase.
    Note that the more overparametrized the model, the higher $D(N, \Delta N)$ is.
    }
    \label{fig:maml-overfit}
\end{figure}

\paragraph{Upwards generalization and path-independence}
Our results might seem to be in contradiction with the empirical findings reported by \citet{anil2022path}.
They study how the test-time performance is affected by an increase in the number of inner iterations at inference, but unlike us consider harder problems.
They correlate the capacity of DEQs to benefit from more test-time inner iterations with a property termed path-independence.
A DEQ is said to be path-independent if for a given couple $\theta, x_i$, there exists only one root of $f(z, \theta, x_i)$.

We highlight that our theoretical analysis is also valid for path-independent DEQs, which also suffer from I2O in-distribution.
Typically, when $K_\text{in}^\top = I$ and we have a fully invertible affine DEQ layer, there is only a single root $z^\star(\theta) = -B^{-1}(U\theta + c)$ (see \autoref{ass:aff-inner}) and this is covered by \autoref{thm:quadratic}.
However, in our experiments, we consider in-distribution generalization, rather than upwards generalization, i.e., an out-of-distribution setting where an explicit difficulty parameter is set higher at test-time than during training, which explains why our results and those of \citet{anil2022path} are not in contradiction.


\section{Robustness of the networks obtained with IFT gradient descent to I2O}
\label{sec:stability}

\autoref{thm:iftcvg} shows that the way the hypergradient is computed for implicit deep learning, IFT or unrolling, does not impact whether it suffers from I2O in our simplified setting.
Still, one might wonder whether the effect of I2O is stronger for one or the other in practical cases.
\citet{gilton2021deep} suggest that I2O is much more prevalent for unrolling, highlighting a huge drop in performance when more iterations are used during inference.
However, in their experiment, the network trained with unrolling has an effective depth much smaller than its IFT-trained counterpart --10 fixed-point iterations vs 50 Anderson acceleration iterations.
Indeed, because IFT gradient descent enables a memory-free training, it is possible to train networks with a very large effective depth, while it is harder for unrolled networks.

\begin{figure}
    \centering
    \includegraphics{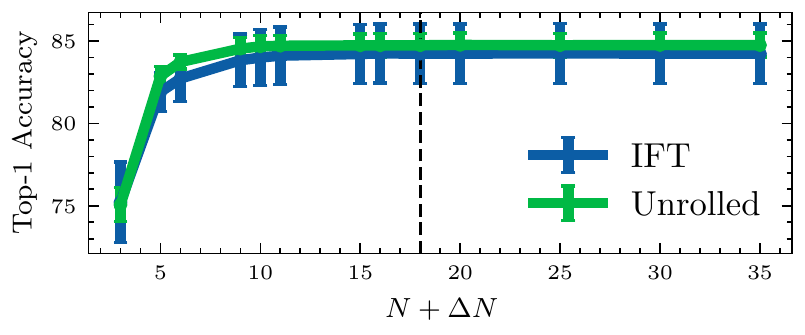}
    \caption{\textbf{Stability of unrolled/IFT trained networks}: We compare how stable networks trained with either \textcolor{xkcdMediumGreen}{unrolling} or \textcolor{xkcdBlue}{IFT} are to the choice of the number of inner iterations at inference. The task is image classification on CIFAR-10. The networks were trained with 18 iterations.
    Note that the IFT trained network does not appear more stable than its unrolled counterpart.
    }
    \label{fig:stability}
\end{figure}

We tried to test whether this conclusion still holds when training networks with unrolling and IFT with the same depth.
In our experiment, both networks use Broyden's method to solve the inner problem.
Therefore, for the unrolled network, we backpropagate through the Broyden iterates which has a high memory requirement.
For this reason, we can only train relatively small networks for image classification on CIFAR-10~\citep{Krizhevsky09learningmultiple}.
In~\autoref{fig:stability} we compare the stability of the two networks for different number of inner iterations averaging performance over 10 seeds.
We observe that there is no gain of stability when using IFT gradient descent over unrolling.


\section{Conclusion}
\label{sec:ccl}

In this work, we challenge one common assumption about DEQs: the possibility to select their computational budget at inference.
We showed that not only do DEQs exhibit a phenomenon we termed inner iterations overfitting~(I2O), that we proved to be grounded in theory for simple models, they also do not appear to have more stability than unrolled networks.
We highlight that this does not mean that DEQs should not be used: the $O(1)$ memory requirement during training is a strong point which enables the training of very deep networks which are not trainable with unrolling.

One big challenge that remains is to understand the tools that could be used to study nonlinear DEQs.
Indeed, the approximation of a nonlinear DEQ with an affine one in the first order is difficult to study because of the integration of errors when solving the fixed point.
Finally, on a more practical note, since we noted that eventually one wants to use DEQs with a fixed number of iterations, a question that remains is: can we bias the hypergradient so that it does take into account the number of iterations made in a more direct way, somehow making it more similar to the unrolled gradient?

\section*{Acknowledgments}
The work of G. Peyré and Z. Ramzi was supported by the European Research Council (ERC project NORIA) and the French government under management of Agence Nationale de la Recherche as part of the ``Investissements d’avenir'' program, reference ANR19-P3IA-0001 (PRAIRIE 3IA Institute).
This work was performed using HPC resources from GENCI–IDRIS (Grant 2022-AD011013570).
We thank Elias Ramzi, Karl Hajjar, Zhenzhang Ye and Matthieu Terris for their valuable feedbacks on a draft of this work.
We also thank Francisco Andrade and Rapha\"{e}l Barboni for fruitful mathematical discussions.
{
\setlength\bibitemsep{3\itemsep}  
\small
\printbibliography

@inproceedings{baiDeepEquilibriumModels2019,
  title = {Deep {{Equilibrium Models}}},
  booktitle = {Neural {{Information Processing Systems}} ({{NeurIPS}})},
  author = {Bai, Shaojie and Kolter, J. Zico and Koltun, Vladlen},
  year = {2019}
}

@inproceedings{baimultiscalemodels2020,
 author = {Bai, Shaojie and Koltun, Vladlen and Kolter, J. Zico},
 booktitle = {{Advances in Neural Information Processing Systems (NeurIPS)}},
 title = {Multiscale Deep Equilibrium Models},
 year = {2020}
}

@InProceedings{Bai_2022_CVPR,
    author    = {Bai, Shaojie and Geng, Zhengyang and Savani, Yash and Kolter, J. Zico},
    title     = {Deep Equilibrium Optical Flow Estimation},
    booktitle = {IEEE/CVF Conference on Computer Vision and Pattern Recognition (CVPR)},
    year      = {2022},
}

@inproceedings{
bai2022neural,
title={Neural Deep Equilibrium Solvers},
author={Shaojie Bai and Vladlen Koltun and J Zico Kolter},
booktitle={International Conference on Learning Representations (ICLR)},
year={2022},
url={https://openreview.net/forum?id=B0oHOwT5ENL}
}

@inproceedings{fung2022jfb,
  title={JFB: Jacobian-free backpropagation for implicit networks},
  author={Fung, Samy Wu and Heaton, Howard and Li, Qiuwei and McKenzie, Daniel and Osher, Stanley and Yin, Wotao},
  booktitle={AAAI Conference on Artificial Intelligence (AAAI)},
  year={2022}
}

@article{paszkePyTorchImperativeStyle2019,
  title={Pytorch: An imperative style, high-performance deep learning library},
  author={Paszke, Adam and Gross, Sam and Massa, Francisco and Lerer, Adam and Bradbury, James and Chanan, Gregory and Killeen, Trevor and Lin, Zeming and Gimelshein, Natalia and Antiga, Luca and others},
  journal={Advances in neural information processing systems (NeurIPS)},
  year={2019}
}

@book{Krantz2013TheApplications,
    title = {{The Implicit Function Theorem: History, Theory, and Applications}},
    year = {2013},
    booktitle = {The Implicit Function Theorem: History, Theory, and Applications},
    author = {Krantz, Steven G. and Parks, Harold R.},
    month = {1},
    pages = {1--163},
    publisher = {Springer New York}
}

@article{gilton2021deep,
  title={Deep equilibrium architectures for inverse problems in imaging},
  author={Gilton, Davis and Ongie, Gregory and Willett, Rebecca},
  journal={IEEE Transactions on Computational Imaging},
  volume={7},
  pages={1123--1133},
  year={2021},
  publisher={IEEE}
}

@article{hestenes1952methods,
  title={Methods of conjugate gradients for solving linear systems},
  author={Hestenes, Magnus R and Stiefel, Eduard and others},
  journal={Journal of research of the National Bureau of Standards},
  volume={49},
  number={6},
  pages={409--436},
  year={1952}
}

@misc{pedregosa2021residual,
  title={Residual Polynomials and the Chebyshev method},
  author={Pedregosa, Fabian},
  NOTE = {URL: http://fa.bianp.net/blog/2020/polyopt/},
  url = {http://fa.bianp.net/blog/2020/polyopt/},
  year={2020}
}

@book{fischer2011polynomial,
  title={Polynomial based iteration methods for symmetric linear systems},
  author={Fischer, Bernd},
  year={2011},
  publisher={SIAM}
}

@MISC{projexp,
    TITLE = {How to compute the expectation of the $2$-norm of an orthogonally projected vector?},
    AUTHOR = { \href{https://math.stackexchange.com/users/601630/zhm1995}{zhm1995}},
    year={2019},
    HOWPUBLISHED = {Mathematics Stack Exchange},
    NOTE = {URL:https://math.stackexchange.com/q/3248980 (version: 2019-06-03)},
    EPRINT = {https://math.stackexchange.com/q/3248980},
    URL = {https://math.stackexchange.com/q/3248980}
}

@article{zou2023deep,
  title={Deep Equilibrium Learning of Explicit Regularizers for Imaging Inverse Problems},
  author={Zou, Zihao and Liu, Jiaming and Wohlberg, Brendt and Kamilov, Ulugbek S},
  journal={arXiv preprint arXiv:2303.05386},
  year={2023}
}

@InProceedings{pmlr-v162-vicol22a,
  title = 	 {On Implicit Bias in Overparameterized Bilevel Optimization},
  author =       {Vicol, Paul and Lorraine, Jonathan P and Pedregosa, Fabian and Duvenaud, David and Grosse, Roger B},
  booktitle = 	 {International Conference on Machine Learning (ICML)},
  year = 	 {2022},
}

@article{rajeswaran2019meta,
  title={Meta-learning with implicit gradients},
  author={Rajeswaran, Aravind and Finn, Chelsea and Kakade, Sham M and Levine, Sergey},
  journal={{Advances in Neural Information Processing Systems (NeurIPS)}},
  year={2019}
}

@Unpublished{micaelli2023recurrence,
  title={Recurrence without Recurrence: Stable Video Landmark Detection with Deep Equilibrium Models},
  author={Micaelli, Paul and Vahdat, Arash and Yin, Hongxu and Kautz, Jan and Molchanov, Pavlo},
  year={2023}
}

@inproceedings{
ramzi2022shine,
title={{SHINE}: {SH}aring the {IN}verse Estimate from the forward pass for bi-level optimization and implicit models},
author={Zaccharie Ramzi and Florian Mannel and Shaojie Bai and Jean-Luc Starck and Philippe Ciuciu and Thomas Moreau},
booktitle={International Conference on Learning Representations (ICLR)},
year={2022},
}

@InProceedings{pmlr-v70-finn17a,
  title = 	 {Model-Agnostic Meta-Learning for Fast Adaptation of Deep Networks},
  author =       {Chelsea Finn and Pieter Abbeel and Sergey Levine},
  booktitle = 	 {International Conference on Machine Learning (ICML)},
  year = 	 {2017},
}

@inproceedings{deng2009imagenet,
  title={Imagenet: A large-scale hierarchical image database},
  author={Deng, Jia and Dong, Wei and Socher, Richard and Li, Li-Jia and Li, Kai and Fei-Fei, Li},
  booktitle={IEEE Conference on Computer Vision and Pattern Recognition (CVPR)},
  year={2009},
}

@TECHREPORT{Krizhevsky09learningmultiple,
    author = {Alex Krizhevsky},
    title = {Learning multiple layers of features from tiny images},
    institution = {},
    year = {2009}
}

@inproceedings{Cordts2016Cityscapes,
title={The Cityscapes Dataset for Semantic Urban Scene Understanding},
author={Cordts, Marius and Omran, Mohamed and Ramos, Sebastian and Rehfeld, Timo and Enzweiler, Markus and Benenson, Rodrigo and Franke, Uwe and Roth, Stefan and Schiele, Bernt},
booktitle={IEEE Conference on Computer Vision and Pattern Recognition (CVPR)},
year={2016}
}

@inproceedings{Butler:ECCV:2012,
title = {A naturalistic open source movie for optical flow evaluation},
author = {Butler, D. J. and Wulff, J. and Stanley, G. B. and Black, M. J.},
booktitle = {European Conf. on Computer Vision (ECCV)},
year = {2012}
}

@inproceedings{
merity2017pointer,
title={Pointer Sentinel Mixture Models},
author={Stephen Merity and Caiming Xiong and James Bradbury and Richard Socher},
booktitle={International Conference on Learning Representations (ICLR)},
year={2017},
}

@inproceedings{martin2001database,
  title={A database of human segmented natural images and its application to evaluating segmentation algorithms and measuring ecological statistics},
  author={Martin, David and Fowlkes, Charless and Tal, Doron and Malik, Jitendra},
  booktitle={Proceedings Eighth IEEE International Conference on Computer Vision. (ICCV)},
  year={2001},
}

@phdthesis{baiequilibrium,
  title={Equilibrium Approaches to Modern Deep Learning},
  author={Bai, Shaojie},
  school={Carnegie Mellon University},
  year={2022}
}

@article{broyden1965class,
  title={A class of methods for solving nonlinear simultaneous equations},
  author={Broyden, Charles G},
  journal={Mathematics of computation},
  volume={19},
  number={92},
  pages={577--593},
  year={1965}
}

@article{chen2020modular,
  title={Modular meta-learning with shrinkage},
  author={Chen, Yutian and Friesen, Abram L and Behbahani, Feryal and Doucet, Arnaud and Budden, David and Hoffman, Matthew and de Freitas, Nando},
  journal={{Advances in Neural Information Processing Systems (NeurIPS)}},
  year={2020}
}

@software{jax2018github,
  author = {James Bradbury and Roy Frostig and Peter Hawkins and Matthew James Johnson and Chris Leary and Dougal Maclaurin and George Necula and Adam Paszke and Jake Vander{P}las and Skye Wanderman-{M}ilne and Qiao Zhang},
  title = {{JAX}: composable transformations of {P}ython+{N}um{P}y programs},
  url = {http://github.com/google/jax},
  version = {0.3.13},
  year = {2018},
}

@inproceedings{jaxopt_implicit_diff,
  title={Efficient and modular implicit differentiation},
  author={Blondel, Mathieu and Berthet, Quentin and Cuturi, Marco and Frostig, Roy and Hoyer, Stephan and Llinares-L{\'o}pez, Felipe and Pedregosa, Fabian and Vert, Jean-Philippe},
  booktitle={Advances in neural information processing systems (NeurIPS)},
  year={2022}
}

@inproceedings{kingma2014adam,
  title={Adam: A method for stochastic optimization},
  author={Kingma, Diederik P and Ba, Jimmy},
  booktitle={International Conference on Learning Representations (ICLR)},
  year={2015}
}

@article{chen2018neural,
  title={Neural ordinary differential equations},
  author={Chen, Ricky TQ and Rubanova, Yulia and Bettencourt, Jesse and Duvenaud, David K},
  journal={{Advances in Neural Information Processing Systems (NeurIPS)}},
  year={2018}
}

@inproceedings{amos2017optnet,
  title={Optnet: Differentiable optimization as a layer in neural networks},
  author={Amos, Brandon and Kolter, J Zico},
  booktitle={International Conference on Machine Learning (ICML)},
  year={2017},
}

@article{anderson,
author = {Anderson, Donald G.},
title = {Iterative Procedures for Nonlinear Integral Equations},
year = {1965},
issue_date = {Oct. 1965},
publisher = {Association for Computing Machinery},
address = {New York, NY, USA},
volume = {12},
number = {4},
issn = {0004-5411},
journal = {J. ACM},
month = {oct},
pages = {547–560},
numpages = {14}
}

@article{anil2022path,
  title={Path Independent Equilibrium Models Can Better Exploit Test-Time Computation},
  author={Anil, Cem and Pokle, Ashwini and Liang, Kaiqu and Treutlein, Johannes and Wu, Yuhuai and Bai, Shaojie and Kolter, J Zico and Grosse, Roger B},
  journal={{{Advances in Neural Information Processing Systems (NeurIPS)}}},
  year={2022}
}

@article{pal2022continuous,
  title={Continuous Deep Equilibrium Models: Training Neural ODEs faster by integrating them to Infinity},
  author={Pal, Avik and Edelman, Alan and Rackauckas, Christopher},
  journal={arXiv preprint arXiv:2201.12240},
  year={2022}
}

@article{li2018learning,
  title={Learning overparameterized neural networks via stochastic gradient descent on structured data},
  author={Li, Yuanzhi and Liang, Yingyu},
  journal={{Advances in Neural Information Processing Systems (NeurIPS)}},
  year={2018}
}

@inproceedings{du2018gradient,
  title={Gradient descent finds global minima of deep 774 neural networks},
  author={Du, SS and Lee, JD and Li, H and Wang, L and Zhai, X},
  booktitle={International Conference on Learning Representations (ICLR)},
  year={2018}
}

@article{arora2019exact,
  title={On exact computation with an infinitely wide neural net},
  author={Arora, Sanjeev and Du, Simon S and Hu, Wei and Li, Zhiyuan and Salakhutdinov, Russ R and Wang, Ruosong},
  journal={{Advances in Neural Information Processing Systems (NeurIPS)}},
  year={2019}
}

@inproceedings{kawaguchi2021theory,
  title={On the theory of implicit deep learning: Global convergence with implicit layers},
  author={Kawaguchi, Kenji},
  booktitle={International Conference on Learning Representations (ICLR)},
  year={2021}
}
}

\newpage
\appendix

\def\thefigure{\thesection.\arabic{figure}}
\counterwithin{figure}{section}


\section{Proof of main result and corollary}
\label{app:proof-main}
In order to prove the main result, we will rely on the notion of time-invertible linear procedures.
\begin{definition}{(Time-invertible linear procedure)}
    A time-invertible linear procedure is a sequence $z_N(\theta)$ such that there exists $N_0 < \infty, K_\text{in} \in \sR^{d_x \times d_z}, U \in \sR^{d_x \times d_\theta}, E_N \in \sR^{d_x \times d_x}, r_N \in \sR^{d_z}$ such that it can be written as:
    \begin{equation}
        \label{eq:time-invertible-proc}
        z_N(\theta) = K_\text{in}^\top E_N U \theta + r_N,
    \end{equation}
    and $\forall N \geq N_0$, where $E_N$ is invertible.
\end{definition}

Let us now move on to the main components of the proof.
We need first to get an expression for the iterates of the fixed-point iterations for an affine function.
\begin{lemma}
    \label{lem:fpi-it}
    Let us assume $f(z) = K^\top(Bz+ c)$, with $K, B \in \sR^{d_x \times d_x}$ and $c \in \sR^{dx}$.
    We further assume that $BK^\top$ has eigenvalues with positive real part and that $K$ is surjective.
    Then the iterates of the fixed-point iteration method with fixed step size $\eta$ have the following expression:
    \begin{equation}
        z_N = K^\top \left(\left(I - \eta B K^\top \right)^N -I \right)  (B K^\top)^{-1} (Bz_0 + c) + z_0,
    \end{equation}
    with $z_0$ the initialization of the procedure.
\end{lemma}

\begin{proof}
    We have:
    \begin{align}
         z_{N+1} &= z_N - \eta K^\top(Bz_N + c)\\
         &= (I - \eta K^\top B)z_N - \eta K^\top c
    \end{align}
    First let us check that $z_N - z_0$ is in the range of $K^\top$ for all $N$.
    We proceed by recurrence starting with $N=0$, which is obviously true because $z_0 - z_0 = 0$ is the range of $K^\top$.
    If $z_N - z_0$ is in the range of $K^\top$, there exists $x$ such that $z_N - z_0 = K^\top x$.
    Then we have \begin{align}
        z_{N+1} &= z_N - \eta K^\top B z_N - \eta K^\top c \\
        z_{N+1} - z_0 &= z_N - z_0 - \eta K^\top B z_N - \eta K^\top c \\
        z_{N+1} - z_0 &= K^\top x - \eta K^\top B z_N - \eta K^\top c \\
        z_{N+1} - z_0 &= K^\top \left(x - \eta  B z_N - \eta  c \right)
    \end{align}
    Therefore $z_{N+1}$ is also in the range of $K^\top$ and we conclude that $z_N - z_0$ is in the range of $K^\top$ for all $N$.
    We then introduce $y_N$ such that $z_N - z_0 = K^\top y_N$.
    The following recurrence then holds:
    \begin{align}
        z_{N+1} - z_0 &= z_N - z_0 - \eta K^\top B z_N - \eta K^\top c \\
        K^\top y_{N+1} &= K^\top y_N - \eta K^\top B z_N - \eta K^\top c \\
        K^\top y_{N+1} &= K^\top \left(y_N - \eta B z_N - \eta  c \right) \\
        y_{N+1} &= y_N - \eta B z_N - \eta  c  \\
        y_{N+1} &= y_N - \eta B (z_N - z_0) - \eta B z_0 - \eta  c  \\
        y_{N+1} &= y_N - \eta B K^\top y_N - \eta B z_0 - \eta  c  \\
        y_{N+1} &= \left(I - \eta B K^\top \right) y_N - \eta B z_0 - \eta  c
    \end{align}
    The expression of $y_N$ is then for $y_0 = 0$:
    \begin{align}
        y_N &= -\left(I - \eta B K^\top \right)^N  (B K^\top)^{-1} (Bz_0 - c) - (B K^\top)^{-1} (Bz_0 + c)\\
        &= \left(\left(I - \eta B K^\top \right)^N -I \right)  (B K^\top)^{-1} (Bz_0 + c)
    \end{align}
    Therefore the expression of $z_N$ is:
    \begin{align}
        z_N - z_0  &= K^\top \left(\left(I - \eta B K^\top \right)^N -I \right)  (B K^\top)^{-1} (Bz_0 + c) \\
        z_N &= K^\top \left(\left(I - \eta B K^\top \right)^N -I \right)  (B K^\top)^{-1} (Bz_0 + c) + z_0
    \end{align}
\end{proof}

\begin{lemma}
\label{lem:tilp}
    Under \autoref{ass:fpi} and \autoref{ass:aff-inner}, the inner procedure (i.e. the fixed-point iteration method) is a time-invertible linear procedure with $E_N = \left((I - \eta B K_\text{in}^\top )^N - I\right) (B K_\text{in}^\top )^{-1}$ and $r_N = K_\text{in}^\top E_N  (Bz_0 + c) + z_0$, with $z_0$ the initialization of the procedure.
\end{lemma}

\begin{proof}
    Using \autoref{lem:fpi-it}, we have the expression of $z_N(\theta)$:
    \begin{align}
        z_N(\theta) &= K_\text{in}^\top \left((I - \eta B K_\text{in}^\top )^N - I\right) (B K_\text{in}^\top )^\dagger  (U \theta + c + Bz_0) +  z_0\\
        &= K_\text{in}^\top \left((I - \eta B K_\text{in}^\top )^N - I\right) (B K_\text{in}^\top )^\dagger  U \theta \\&\hphantom{=}\enspace+  K_\text{in}^\top \left((I - \eta B K_\text{in}^\top )^N - I\right) (B K_\text{in}^\top )^\dagger (c + Bz_0) + z_0\\
        &= K_\text{in}^\top E_N U \theta + r_N,
    \end{align}
    with $E_N = \left((I - \eta B K_\text{in}^\top )^N - I\right) (B K_\text{in}^\top )^{-1}$ and $r_N = K_\text{in}^\top E_N  (c + Bz_0) + z_0$.
    We now need to prove that $E_N$ is invertible for $N$ sufficiently large.
    Since the fixed-point iterations converge, $(I - \eta B K_\text{in}^\top )^N $ goes to 0, and therefore $\left((I - \eta B K_\text{in}^\top )^N - I\right)$ goes to $-I$ which means that for $N$ sufficiently large the latter is invertible.
    We conclude by noticing that $E_N$ is invertible as the product of two invertible matrices.
\end{proof}

We restate the main theorem here for convenience, before proving it.
\quadratic*

\begin{proof}
    We prove here the result directly for a time-invertible linear procedure thanks to \autoref{lem:tilp}:
\begin{equation}
    z_N(\theta) = K_\text{in}^\top E_N U \theta + r_N = A_N \theta + r_N, \text{ with } A_N = K_\text{in}^\top E_N U 
\end{equation}
We can then write:
\begin{align}
    \ell(z_N(\theta)) &= \frac12 \|K_\text{out} z_N(\theta) - \omega\|_2^2 \\
    &= \frac12 \|K_\text{out} A_N \theta + K_\text{out} r_N - \omega\|_2^2 \\
    &= \frac12 \|K_\text{out} A_N \theta + \mathcal{P}(K_\text{out} A_N ) (K_\text{out} r_N - \omega)\|_2^2 + \frac12 \| \mathcal{P}(\left(K_\text{out} A_N \right)^\perp) (K_\text{out} r_N - \omega) \|_2^2
\end{align}
where the last equation holds because the two terms are orthogonal.
And we have:
\begin{align}
    \label{eq:K_grad_l}
    &K_\text{out} K_\text{in}^\top E_N U \theta^{\star, N} =  -\mathcal{P}(K_\text{out} A_N ) (K_\text{out} r_N - \omega)\\
    & E_N U \theta^{\star, N} = -(K_\text{out} K_\text{in}^\top)^\dagger  \mathcal{P}(K_\text{out} A_N ) (K_\text{out} r_N - \omega) + \theta_{\ker(K_\text{out} K_\text{in}^\top)} \\
    & U \theta^{\star, N} = -E_N^{-1}(K_\text{out} K_\text{in}^\top)^\dagger  \mathcal{P}(K_\text{out} A_N ) (K_\text{out} r_N - \omega) + E_N^{-1}\theta_{\ker(K_\text{out} K_\text{in}^\top)}
\end{align}

Plugging that into the outer loss for a different inner iterations number $N' = N + \Delta N$ and using $C = K_\text{out} K_\text{in}^\top$, $M(N', N) = E_{N'} E_N^{-1}$ and $v_N = (K_\text{out} r_N - \omega)$:
\begin{align}
    \ell(z_{N'}(\theta^{\star, N})) &= \frac12 \|C E_{N'} U \theta^{\star, N} + K_\text{out} r_{N'} - \omega\|_2^2 \\
    &= \frac12 \|- C M(N', N) C^\dagger  \mathcal{P}(K_\text{out} A_N ) v_N + v_{N'} + C M(N', N) \theta_{\ker(C)}\|_2^2 \\
    &= \frac12 \|- C M(N', N) C^\dagger  \mathcal{P}(K_\text{out} A_N ) v_N + \mathcal{P}(C) v_{N'} + C M(N', N) \theta_{\ker(C)}\|_2^2 \\
    &\hphantom{=}\enspace+ \frac12 \|\mathcal{P}(C^\perp) v_{N'}\|_2^2
\end{align}

For $N' = N$, i.e. $\Delta N=0$, since $M(N', N) = I$, we have:

\begin{align}
    \ell(z_N(\theta^{\star, N})) &= \frac12 \|- C  C^\dagger  \mathcal{P}(K_\text{out} A_N ) v_N + \mathcal{P}(C) v_N + C  \theta_{\ker(C)}\|_2^2 + \frac12 \|\mathcal{P}(C^\perp) v_N\|_2^2 \\
    &= \frac12 \|- C  C^\dagger  \mathcal{P}(K_\text{out} A_N ) v_N + \mathcal{P}(C) v_N\|_2^2 + \frac12 \|\mathcal{P}(C^\perp) v_N\|_2^2 \\
    &= \frac12 \|- \mathcal{P}(C)  \mathcal{P}(C E_N U ) v_N + \mathcal{P}(C) v_N\|_2^2 + \frac12 \|\mathcal{P}(C^\perp) v_N\|_2^2 \\
    &= \frac12 \|- \mathcal{P}(C E_N U ) v_N + \mathcal{P}(C) v_N\|_2^2 + \frac12 \|\mathcal{P}(C^\perp) v_N\|_2^2 \\
    &= \frac12 \|\left(\mathcal{P}(C)- \mathcal{P}(C E_N U )\right) v_N \|_2^2 + \frac12 \|\mathcal{P}(C^\perp) v_N\|_2^2
\end{align}
The above expressions are equalities involving complicated terms. In order to have a simpler formula, we get a lower bound via the following:
\begin{align}
    D(N, \Delta N) &= \ell(z_{N + \Delta N}(\theta^{\star, N})) - \ell(z_N(\theta^{\star, N})) \\
    & \geq - \frac12 \|\left(\mathcal{P}(C)- \mathcal{P}(C E_N U )\right) v_N \|_2^2\\
    & \geq - \frac12 \|\left(\mathcal{P}(K_\text{out} K_\text{in}^\top)- \mathcal{P}(K_\text{out} K_\text{in}^\top E_N U )\right) (K_\text{out} r_N - \omega) \|_2^2
\end{align}
\end{proof}

\begin{lemma}
    For $v \in \sR^{p}$, and $\mathcal{P}(X)$ the orthogonal projection onto $\mathrm{range}(X)$ for $X \in \sR^{p \times d}$, we have
    \begin{equation}
        \E_{X | X_{ij} \sim \mathcal{N}(0, 1)} \|\mathcal{P}(X)v\|_2^2 = \frac{d}{p} \|v\|_2^2.
    \end{equation}
\end{lemma}

\begin{proof}
    The proof is taken from \citet{projexp}.
Let $X = UDV^\top$ a SVD of $X$, with $U\in\sR^{d\times d}$ orthogonal, $D$ diagonal with positive entries, and $V\in \sR^{p\times d}$ orthogonal, i.e. such that $V^\top V = I_d$.
Since the entries of $X$ are drawn i.i.d. from a normal distribution, the singular values of $X$ are non-zeros with probability $1$, and we have $\mathcal{P}(X) = X^\top(XX^\top)^{-1}X = VV^{\top}$.
Since the distribution of $V$ is right invariant, the distribution of $\|\mathcal{P}(X)v\|_2^2$
 depends only on $\|v\|_2^2$.\footnote{Right invariance means, for each fixed $p\times p$ orthogonal matrix $B$, the matrix $VB$ is distributed the same way $V$ is. It is a consequence of the rotational symmetry of the original normal distribution.}
 Thus, $\E_{X | X_{ij} \sim \mathcal{N}(0, 1)} \|\mathcal{P}(X)v\|_2^2 = \|v\|_2^2 \E_{X | X_{ij} \sim \mathcal{N}(0, 1)} \mathcal{P}(X)_{11}$, taking $v = (\|v\|, 0, \ldots, 0)$ (i.e. we chose a basis where the first vector is 0, and if $v$ is 0 the problem is trivial).
 By symmetry, we have $\E_{X | X_{ij} \sim \mathcal{N}(0, 1)} \mathcal{P}(X)_{11} = \E_{X | X_{ij} \sim \mathcal{N}(0, 1)} \operatorname{tr}(\mathcal{P}(X)) / p$.
 And we have $\operatorname{tr}(\mathcal{P}(X)) = \operatorname{tr}(V V^\top ) = \operatorname{tr}(V^\top V) = d$.
\end{proof}

\avgcase*

\begin{proof}
From \autoref{thm:quadratic}, we have for $N \geq N_0$:
\begin{equation}
    D(N, \Delta N) \geq -\frac12 \|\left(\mathcal{P}(K_\text{out} K_\text{in}^\top) - \mathcal{P}(K_\text{out} K_\text{in}^\top E_N U)\right) (K_\text{out} r_N - \omega)\|_2^2
\end{equation}

When the outer problem is strongly convex, we have $K_\text{out}$ is an invertible matrix.
Therefore, $\mathcal{P}(K_\text{out} K_\text{in}^\top) = I$, and $K_\text{out} K_\text{in}^\top E_N = B_N$ is invertible for $N \geq N_0$.
Using the notation $v_N = (K_\text{out} r_N - \omega)$, if $d_x<d_\theta$, we have:
\begin{align}
    \E_{U \sim \mathcal{N}(0, I)} D(N, \Delta N) &\geq -\frac12 \E_{U \sim \mathcal{N}(0, I)} \|\left(\mathcal{P}(K_\text{out} K_\text{in}^\top) - \mathcal{P}(K_\text{out} K_\text{in}^\top E_N U)\right) (K_\text{out} r_N - \omega)\|_2^2\\
    &\geq -\frac12 \E_{U \sim \mathcal{N}(0, I)} \|\left(I - \mathcal{P}(B_N U)\right) v_N \|_2^2\\
    &\geq -\frac12 \E_{U \sim \mathcal{N}(0, I)} v_N^\top \left(I - \mathcal{P}(B_N U)\right)^\top \left(I - \mathcal{P}(B_N U)\right) v_N\\
    &\geq -\frac12 \E_{U \sim \mathcal{N}(0, I)} v_N^\top \left(I - \mathcal{P}(B_N U)\right) \left(I - \mathcal{P}(B_N U)\right) v_N\\
    &\geq -\frac12 \E_{U \sim \mathcal{N}(0, I)} v_N^\top \left(I - \mathcal{P}(B_N U)\right) v_N\\
    &\geq -\frac12 \E_{U \sim \mathcal{N}(0, I)} v_N^\top v_N - v_N^\top \mathcal{P}(B_N U) v_N\\
    &\geq -\frac12 \E_{U \sim \mathcal{N}(0, I)} \|v_N\|_2^2 - v_N^\top \mathcal{P}(B_N U) v_N\\
    &\geq -\frac12 \E_{U \sim \mathcal{N}(0, I)} \|v_N\|_2^2 - v_N^\top \mathcal{P}(B_N U) \mathcal{P}(B_N U) v_N\\
    &\geq -\frac12 \E_{U \sim \mathcal{N}(0, I)} \|v_N\|_2^2 - v_N^\top \mathcal{P}(B_N U)^\top \mathcal{P}(B_N U) v_N\\
    &\geq -\frac12 \E_{U \sim \mathcal{N}(0, I)} \|v_N\|_2^2 - \|\mathcal{P}(B_N U) v_N\|_2^2\\
    &\geq -\frac12 \E_{U' \sim \mathcal{N}(0, I)} \|v_N\|_2^2 - \|\mathcal{P}(U') v_N\|_2^2\\
    &\geq -\frac12  \|v_N\|_2^2 - \E_{U' \sim \mathcal{N}(0, I)} \|\mathcal{P}(U') v_N\|_2^2\\
    &\geq -\frac12  \|v_N\|_2^2 -  \frac{d_\theta}{d_x} \| v_N\|_2^2\\
    &\geq -\frac12  \left(1 -  \frac{d_\theta}{d_x}\right) \| v_N\|_2^2
\end{align}

We can conclude using the triangular inequality and the definition of the spectral radius on $\| v_N\|_2^2$, and using \autoref{cor:overparam} to go from $\frac{d_\theta}{d_x}$ to $\frac{\min(d_x, d_\theta)}{d_x}$.
\end{proof}





\section{How relevant is the affine inner problem factorization?}
\label{app:affine-inner}

In order to derive our analysis, we assumed in \autoref{ass:aff-inner} that the expression of the inner problem root-defining function followed a certain factorization $f(z, \theta) = K_\text{in}^\top(Bz + U\theta + c)$.
We show that this form is satisfied by two classes of problems: affine DEQs and meta-learning a linear model.

We first tackle the case of affine DEQs.
To do so, let us first establish the following lemma:
\begin{lemma}
    \label{lemma:fixed-point-iteration-fact}
    If $f$ is an affine function of the form $f(z) = Az + c$, and the fixed-point iteration method with fixed step size $\eta$ to find its root converges for any initialization $z_0$, then $f$ can be factorized, i.e. $\exists K, \Gamma, \gamma$ such that $f(z) = K^\top\Gamma z + K^\top\gamma$ and $K$ is surjective.
    Moreover, $\Gamma K^\top$ has eigenvalues with positive real part.
\end{lemma}

\begin{proof}
    If the fixed-point iteration method converges, then it means that $f$ has a root denoted $z^\star$.
    This root verifies $Az^\star = -c$, so it means that $c$ is in the range of $A$.
    Furthermore, we can write the low-rank factorization of $A$ as $K^\top \Gamma$ with $K$ surjective.
    Because $c$ is in the range of $A$ it is also in the range of $K^\top$.
    We can denote $c = K^\top \gamma$.
    Using \autoref{lem:fpi-it}, we have the expression of $z_N$:
    \begin{equation}
        z_N = K^\top \left(\left(I - \eta \Gamma K^\top \right)^N -I \right)  (\Gamma K^\top)^{-1} (\Gamma z_0 + \gamma) + z_0
    \end{equation}      
    Since $z_N$ converges for any $z_0$, this means that the largest eigenvalue of $(I - \eta \Gamma K^\top)$ is bounded by 1 in magnitude.
    We can choose $\eta$ as $\frac{1}{\lambda_\text{max} + \epsilon}$ to realize this if all the eigenvalues of $\Gamma K^\top$ have a positive real part.
\end{proof}

Similarly we have the following result:
\begin{restatable}[]{prop}{lineardeq}
    \label{prop:lineardeq}Let us assume that $f$ is affine in $z$ and $\theta$.
    If the fixed-point iteration method with fixed step size $\eta$ converges for $f$, then it satisfies \autoref{ass:aff-inner}.
\end{restatable}


We now move on to meta-learning a linear model with quadratic loss which we show to be a special case of the above.
\begin{restatable}[]{prop}{linearmaml}
    \label{prop:linearmaml}.
    If the task specific regularized loss for iMAML is a convex qudratic function and can be written as:
    \begin{equation}
        F(z, \theta) = \frac12 \| Xz - y \|_2^2 + \lambda \| z - \theta\|_2^2,
    \end{equation}
    with $\mathcal{X}_\text{train} = (X, y)$ the task training set, and $\lambda$ the meta regularization parameter, and gradient descent with fixed step size $\eta$ converges to minimize $F$ in $z$, then $\nabla_z F$ satisfies \autoref{ass:aff-inner}.
\end{restatable}

\begin{proof}
    Since $F$ is a quadratic function of $z$ and $\theta$, $\nabla F$ is an affine function and the gradient descent on $F$ with fixed step size $\eta$ corresponds to the fixed-point iteration method for $\nabla F$.
    We can conclude using \autoref{prop:lineardeq}
\end{proof}


\section{Implicit Differentiation proofs}
\label{app:proof-ift}

\iftcvg*

\begin{proof}
Let us borrow the notations of \autoref{prop:quad-proc}.
Let us denote $H = K_\text{in}^\top B$, $\bar{H} = B K_\text{in}^\top$ and $G = K_\text{out}^\top K_\text{out}$.
We have:
\begin{align}
    p_N(\theta) &= - (H^\dagger K_\text{in}^\top U)^\top \nabla l (z_N(\theta))\\
    &= -(H^\dagger K_\text{in}^\top U)^\top (Gz_N(\theta) - K_\text{out}^\top\omega)\\
    &= -(H^\dagger K_\text{in}^\top U)^\top (G(K_\text{in}^\top E_NU\theta + r_N) - K_\text{out}^\top\omega)\\
    &= -(H^\dagger K_\text{in}^\top U)^\top G K_\text{in}^\top E_NU\theta -(H^\dagger K_\text{in}^\top U )^\top (Gr_N - K_\text{out}^\top\omega)\\
    &= -(\bar{H}^{-1} U)^\top K_\text{in} G K_\text{in}^\top E_NU\theta -(H^\dagger K_\text{in}^\top U )^\top (Gr_N - K_\text{out}^\top\omega)\\
    &= X_N \theta - b_N
\end{align}
with $X_N = -(\bar{H}^{-1} U)^\top K_\text{in} G K_\text{in}^\top E_NU$ and $b_N = (H^\dagger K_\text{in}^\top U )^\top (Gr_N - K_\text{out}^\top\omega)$.
The affine dynamical system we need to study is therefore:
\begin{equation}
    \theta^{T+1, N}_\text{IFT} = (I - \alpha_N X_N) \theta^{T, N}_\text{IFT} - \alpha_N b_N
\end{equation}

If $X_N$ has only nonnegative eigenvalues real part, then we can use $\alpha_N = \frac{1}{\lambda_\text{max} + \epsilon}$ where $\lambda_\text{max}$ is the largest eigenvalue module of $X_N$ and $\epsilon > 0$.
In this case the largest real part of an eigenvalue of $(I - \alpha_N X_N)$ is bounded in magnitude by $1$ and the dynamical system converges.

We now need to show that $X_N$ has only nonnegative eigenvalues real part.
To do so, let's write $X_N$ as the difference between a symmetric and a non-symmetric matrix, using $P_N(\bar{H}) = (I - \eta \bar{H})$:
\begin{align}
    X_N &= -(\bar{H}^{-1} U)^\top K_\text{in} G K_\text{in}^\top E_NU\\
     &= -(\bar{H}^{-1} U)^\top K_\text{in} G K_\text{in}^\top (P_N(\bar{H}) - I) \bar{H}^{-1} U\\
     &= (\bar{H}^{-1} U)^\top K_\text{in} G K_\text{in}^\top \bar{H}^{-1} U - (\bar{H}^{-1} U)^\top K_\text{in} G K_\text{in}^\top P_N(\bar{H}) \bar{H}^{-1} U\\
     &= X_{\text{sym}} - X_{N, \text{non-sym}}
\end{align}

If we take one unit eigenvector $v$  of $X_N$ with associated eigenvalue $\lambda$ we have:
\begin{align}
    \lambda &= \|v\|^2 \lambda \\
    &= v^\top \lambda v \\
    &= v^\top (X_{\text{sym}} - X_{N, \text{non-sym}}) v\\
    &= v^\top X_{\text{sym}} v - v^\top X_{N, \text{non-sym}} v\\
    &= \|K_\text{out} K_\text{in}^\top \bar{H}^{-1} U v\|^2_2 - v^\top (\bar{H}^{-1} U)^\top K_\text{in} G K_\text{in}^\top P_N(\bar{H}) \bar{H}^{-1} U v\\
    &= \|K_\text{out} K_\text{in}^\top y\|^2_2 - y^\top K_\text{in} G K_\text{in}^\top P_N(\bar{H}) y\\
    &= \|K_\text{out} \gamma\|^2_2 - \gamma^\top K_\text{out}^\top K_\text{out}  P_N(H) \gamma
\end{align}
with $y = \bar{H}^{-1} U$ and $\gamma = K_\text{in}^\top y$.

We can first notice that for $\gamma \in \ker(K_\text{out})$, $\lambda = 0$.

We now consider $\gamma \not \in \ker(K_\text{out})$.

\begin{equation}
    \gamma^\top K_\text{out}^\top K_\text{out}  P_N(H) \gamma = \langle \gamma, P_N(H) \gamma \rangle_{K_\text{out}} \leq \|P_N(H)\|_{K_\text{out}} \|K_\text{out} \gamma\|^2_2
\end{equation}

Therefore, we have:
\begin{equation}
    \lambda \geq (1 - \|P_N(H)\|_{K_\text{out}})\|K_\text{out} \gamma\|^2_2
\end{equation}

Because the inner procedure is a converging fixed-point iteration method,  $|P_N(\lambda)| $ can be made arbitrarily small for all eigenvalues of $H$.
In particular, $\exists N_0$ such that $\forall N > N_0$,  $\|P_N(H)\|_{K_\text{out}} < 1 $.
Therefore, $\lambda \geq 0$.

This proof generalizes to gradient-based methods by replacing $P_N$ with the associated residual polynomial.

\end{proof}

\iftequivunrolled*

\begin{proof}
    We borrow the notations from the proof of \autoref{prop:quad-proc}.
    Let us rewrite the root problem satisfied by $\theta^{\star, N}_{\text{IFT}}$:

    \begin{align}
        &(H^\dagger K_\text{in}^\top U)^\top \nabla l (z_N(\theta)) = 0\\
        &\Leftrightarrow (H^\dagger K_\text{in}^\top U)^\top (Gz_N(\theta) - K_\text{out}^\top\omega) = 0\\
        &\Leftrightarrow (H^\dagger K_\text{in}^\top U)^\top (G(K_\text{in}^\top E_NU\theta + r_N) - K_\text{out}^\top\omega) = 0\\
        &\Leftrightarrow (H^\dagger K_\text{in}^\top U)^\top G K_\text{in}^\top E_NU\theta + (H^\dagger K_\text{in}^\top U)^\top(G r_N - K_\text{out}^\top\omega) = 0\\
        &\Leftrightarrow (H^\dagger K_\text{in}^\top U)^\top G K_\text{in}^\top E_NU\theta =-  (H^\dagger K_\text{in}^\top U)^\top(G r_N - K_\text{out}^\top\omega)\\
        &\Leftrightarrow (H^\dagger K_\text{in}^\top U)^\top G K_\text{in}^\top E_NU\theta =-  (H^\dagger K_\text{in}^\top U)^\top K_\text{out}^\top(K_\text{out} r_N - \omega)\\
        &\Leftrightarrow (H^\dagger K_\text{in}^\top U)^\top G K_\text{in}^\top E_NU\theta =-  (K_\text{out} H^\dagger K_\text{in}^\top U)^\top (K_\text{out} r_N - \omega)\\
        &\Leftrightarrow (K_\text{in}^\top \bar{H}^{-1}  U)^\top G K_\text{in}^\top E_NU\theta =-  (K_\text{out} K_\text{in}^\top \bar{H}^{-1}  U)^\top (K_\text{out} r_N - \omega)\\
        &\Leftrightarrow (K_\text{out} K_\text{in}^\top \bar{H}^{-1}  U)^\top K_\text{out} K_\text{in}^\top E_NU\theta =-  (K_\text{out} K_\text{in}^\top \bar{H}^{-1}  U)^\top (K_\text{out} r_N - \omega)\\
        &\Leftrightarrow  K_\text{out} K_\text{in}^\top E_NU\theta =-   (K_\text{out} r_N - \omega) + \theta_{\ker((K_\text{out} K_\text{in}^\top \bar{H}^{-1}  U)^\top)}\\
        &\Leftrightarrow  K_\text{out} K_\text{in}^\top E_NU\theta =-  \mathcal{P}(K_\text{out} A_N) (K_\text{out} r_N - \omega) + \mathcal{P}(K_\text{out} A_N) \theta_{\ker((K_\text{out} K_\text{in}^\top \bar{H}^{-1}  U)^\top)}\\
        &\Leftrightarrow  K_\text{out} K_\text{in}^\top E_NU\theta =-  \mathcal{P}(K_\text{out} A_N) (K_\text{out} r_N - \omega) + \mathcal{P}(K_\text{out} A_N) \theta_{\Ima(K_\text{out} K_\text{in}^\top \bar{H}^{-1}  U)^\perp}\\
        &\Leftrightarrow  K_\text{out} K_\text{in}^\top E_NU\theta =-  \mathcal{P}(K_\text{out} K_\text{in}^\top) (K_\text{out} r_N - \omega) + \mathcal{P}(K_\text{out} K_\text{in}^\top) \theta_{\Ima(K_\text{out} K_\text{in}^\top)^\perp}\\
        &\Leftrightarrow  K_\text{out} K_\text{in}^\top E_NU\theta =-  \mathcal{P}(K_\text{out} K_\text{in}^\top) (K_\text{out} r_N - \omega) \\
        &\Leftrightarrow   U\theta =-  E_N^{-1} (K_\text{out} K_\text{in}^\top)^\dagger \mathcal{P}(K_\text{out} A_N) (K_\text{out} r_N - \omega) + E_N^{-1}  \theta_{\ker(K_\text{out} K_\text{in}^\top)}
    \end{align}
    And we end up with the same characterization as \eqref{eq:K_grad_l}, which concludes the proof.

\end{proof}




\section{Extended related works}
\label{app:related}
\paragraph{Theoretical analysis of implicit deep learning practice}
Relatively few works have looked at how practical implicit deep learning implementations affect the final solution.
Most notable is the one of \citet{pmlr-v162-vicol22a}.
In this work, the authors looked at the implicit biases caused by warm-starting in the inner optimization problem and the use of approximate hypergradients on the final solution.
They proved theoretical results in a quadratic bilevel optimization setting, and showed empirical results on dataset distillation and data augmentation network learning.
Their conclusion is that warm-starting in the inner problem leads to overfitting and information leakage, and that using a higher quality hypergradient leads to min-norm solutions for the outer problem.
Our work is complementary to theirs in that they have not considered non warm-start cases with a fixed number of iterations.
We also highlight that the notion of overparametrization later introduced in our work is different from theirs.
Their notion refers to the inner or outer problem having potentially more than one solution (a setup we cover).

\section{Gradient-based methods and Residual Polynomials}
\label{app:res-poly}
Using the notations of \citet{pedregosa2021residual}, a gradient-based method can be defined as having iterates of the following form:
\begin{equation}
    z_{N+1} = z_N + \sum_{i=0}^{N-1} c_i^{(N)} (z_{i+1} - z_i) + c_i^{(N)} \nabla f(z_N),
\end{equation}
The residual polynomials of this gradient-based method are then defined recursively as:
\begin{equation}
    \begin{array}{l}
        P_{N+1}(\lambda) = (1 + c_N^{(N)} \lambda) P_N(\lambda) + \sum_{i=0}^{N-1} c_i^{(N)} (P_{i+1}(\lambda) - P_i(\lambda)) \\
        P_0(\lambda) = 1
    \end{array}
\end{equation}

\begin{lemma}
    \label{lemma:z_N_discrete}
    Let $F: \sR^{d_z} \to \sR$, $F(z) = \frac12 \|Kz + c\|_2^2 $.
    We define $z_N$ as the $N$-th iterate of a gradient-based method with associated residual polynomial $P_N$~\citep{hestenes1952methods,fischer2011polynomial} (see \autoref{app:res-poly} for a definition) and initial condition $z_0$.
    Then the closed form expression of $z_N$ is:
    \begin{equation}
        z_N =  P_N(K^\top K)z_0 + (P_N(K^\top K) - I)(K^\top K)^\dagger c
    \end{equation}
\end{lemma}

\begin{proof}
    Writing $H = K^\top K$ the hessian of $F$ and $z^\star \in \argmin_z F(z)$, we have the following equality~\citep{hestenes1952methods,fischer2011polynomial,pedregosa2021residual}:
    \begin{equation}
        z_N - z^\star = P_N(H) (z_0 - z^\star)
    \end{equation}
    Rewriting it, leads to:
    \begin{equation}
        z_N = P_N(H) z_0 - (P_N(H) - I) z^\star
    \end{equation}
    And we have that $z^\star$ is such that $\nabla F(z^\star) = Hz^\star + K^\top c = 0$.
    We can take for example $z^\star = -H^\dagger K^\top c$.
    Therefore:
    \begin{equation}
        z_N = P_N(H) z_0 + (P_N(H) - I)H^\dagger K^\top c
    \end{equation}
\end{proof}

\begin{restatable}[]{prop}{quadproc}
    \label{prop:quad-proc}Let us assume that $f$ is the gradient in $z$ of a function $F$ quadratic in $z$ and linear in $\theta$, convex and bounded from below.
    Then any converging gradient-based method minimizing $F$ is a time-invertible linear procedure.
\end{restatable}

\begin{proof}
    Let us give the expression of $F$:
    \begin{equation}
                \label{eq:non-degen-quad}
                F(z, \theta) = \frac12 \|K_\text{in}z + U\theta + c\|_2^2,
    \end{equation}
    with $K_\text{in}\in \sR^{d_x \times d_z}$ surjective (if not we can always reformulate $F$ to have it surjective), $U \in \sR^{d_x \times d_\theta}$, $c \in \sR^{d_z}$.
    From \autoref{lemma:z_N_discrete}, writing $K_\text{in}^\top K_\text{in} = H$, we know that $z_N = P_N(H) z_0 + (P_N(H) - I)H^\dagger (K_\text{in}^\top (U\theta + c))$.
    Rewriting this, with $\bar{H} = K_\text{in} K_\text{in}^\top$:
    \begin{align}
        z_N &= P_N(H) z_0 + (P_N(H) - I)H^\dagger K_\text{in}^\top (U\theta + c) \\
        &= (P_N(H) - I)H^\dagger K_\text{in}^\top U\theta + P_N(H) z_0 + (P_N(H) - I)H^\dagger K_\text{in}^\top c \\
        &= K_\text{in}^\top (P_N(\bar{H}) - I)\bar{H}^\dagger  U\theta + P_N(H) z_0 + (P_N(H) - I)H^\dagger K_\text{in}^\top c \\
        &= K_\text{in}^\top E_N  U\theta + r_N
    \end{align}
    with $E_N = (P_N(\bar{H}) - I)\bar{H}^\dagger$ and $r_N = P_N(H) z_0 + (P_N(H) - I)H^\dagger K_\text{in}^\top c$.
    Because $K_\text{in}$ is surjective $\bar{H}$ is invertible and $\bar{H}^\dagger= \bar{H}^{-1}$ is as well.
    Because $P_N$ is associated with a converging gradient-based method (see more in \autoref{app:res-poly}), $\exists N_0$ such that $\forall N > N_0$, $P_N(\lambda) \neq 1, \forall \lambda \in [\lambda_\text{min}; \lambda_\text{max}]$, $(P_N(\bar{H}) - I)$ is invertible.
    Therefore, $E_N$ is invertible as the product of 2 invertible matrices.
    This concludes the proof.
\end{proof}


\begin{remark}
    \label{remark:decrease-first-order}
    Gradient-based methods have a rate of convergence proportional to $\max_{\lambda \in [0, \lambda_\text{max}]} |P_N(\lambda)|$~\citep{pedregosa2021residual} where $P_N$ is their associated residual polynomial.
    Therefore, for any converging gradient-based method, $\exists N_0 \text{ s.t. } \forall N > N_0, \forall \lambda \in [\lambda_\text{min}; \lambda_\text{max}] |P_N(\lambda)| < 1$.
\end{remark}

\begin{remark}
    \label{remark:gradient-descent}
    For gradient descent with step size $\frac{1 + \epsilon}{\lambda_{\text{max}}}$ with $-1 < \epsilon < 1$, we have $P_N(\lambda) = (1 - \frac{1 + \epsilon}{\lambda_{\text{max}}} \lambda)^N$~\citep{pedregosa2021residual} $\forall N > 0$, which means that $P_N(\lambda) \neq 1$ if $\lambda \in [\lambda_\text{min}; \lambda_\text{max}]$.
    In the case of gradient descent with an appropriate step size, the inequality~\eqref{ineq:main-ineq} is therefore always true for all optimization steps.
\end{remark}

\begin{remark}
    \label{remark:gd-momentum}
    For gradient descent with momentum with admissible parameters as defined by \citet{pedregosa2021residual}[Blog 2], we can reuse computations made to check the convergence to get $N_0$ from \autoref{remark:gradient-descent}.
    For momentum $m$, $N_0$ is such that:
    \begin{equation}
        \label{eq:momentum-convergence}
        m^{\frac{N_0}{2}} (1 + \frac{1 - m}{1 + m N_0}) < 1
    \end{equation}
\end{remark}

\section{Experimental details}
\label{app:exp_details}
\subsection{DEQs}
In all DEQs experiments, except the stability experiment, we reuse the data, architecture, code (in PyTorch~\citep{paszkePyTorchImperativeStyle2019}) and weights of the original works.
The only difference with the original works is that we use a different number of inner steps for the fixed point resolution, and set the tolerance to a value sufficiently low, so that the maximum number of inner steps is always reached ($10^{-7}$ generally).
We list here the links to the original works public GitHub repositories which contain information on how to download the data, the weights and how to perform inference:
\begin{itemize}
    \item Image classification on ImageNet~\citep{deng2009imagenet} and CIFAR~\citep{Krizhevsky09learningmultiple} and Image segmentation on Cityscapes~\citep{Cordts2016Cityscapes}: \href{https://github.com/locuslab/deq/tree/master/MDEQ-Vision}{locuslab/deq/MDEQ-Vision}~\citep{baimultiscalemodels2020}
    \item Language modeling on WikiText~\citep{merity2017pointer}: \href{https://github.com/locuslab/deq/tree/master/DEQ-Sequence}{locuslab/deq/DEQ-Sequence}~\citep{baiDeepEquilibriumModels2019}
    \item Optical Flow Estimation on Sintel~\citep{Butler:ECCV:2012}: \href{https://github.com/locuslab/deq-flow}{locuslab/deq-flow}~\citep{Bai_2022_CVPR}
    \item Single-image Super resolution on CBSD68~\citep{martin2001database}: \href{https://github.com/wustl-cig/ELDER}{wustl-cig/ELDER}~\citep{zou2023deep}
\end{itemize}

\subsection{DEQs stability}
In order to evaluate the stability of DEQs trained with unrolling (i.e. backpropagating through the iterates of Broyden's method), we needed to implement a differentiable version of Broyden's method.
Except for this, the training code and data pipelines are taken from the original work~\citep{baimultiscalemodels2020} for image classification on CIFAR-10~\citep{Krizhevsky09learningmultiple}.
We simply vary the number of inner steps used for the fixed point resolution and the tolerance of the fixed point resolution at test-time.
The networks used were those following the \texttt{TINY} configuration (see  \href{https://github.com/locuslab/deq/blob/master/MDEQ-Vision/experiments/cifar/cls_mdeq_TINY.yaml}{the original configuration file} for more details).

One might notice that Broyden's method is usually not differentiable because of the use of a line search.
As for the typical DEQ setting, we did not use a line search to train DEQs, whether unrolled or not, and just kept a fixed step size of 1, making it differentiable.

\subsection{(i)MAML}
We recall the main difference between MAML and iMAML.
In MAML, the meta parameters are used as an initialization to a gradient descent for task adapted networks, while in iMAML, the meta parameters are used as an anchor point for the task adapted weights.
Formally, for iMAML, the inner loss is modified to include a regularization term that penalizes the $\ell_2$-norm of the difference between the task adapted parameters and the meta-parameters.
Thanks to this formulation, iMAML can use implicit differentiation to compute the hypergradient, while MAML has to rely on unrolling.

We reimplemented the (i)MAML framework in Jax~\citep{jax2018github} using the recently developed Jaxopt library~\citep{jaxopt_implicit_diff}.
The reason for this was that we wanted a faster implementation, made possible by the machinery of Jax and Jaxopt together.

For the sinusoid regression task, we used the same architecture and hyperparameters as the one introduced by \citet{pmlr-v70-finn17a} (respectively the same hyperparameters as the one of \citet{rajeswaran2019meta}, including the use of line search for gradient descent).
Each sinusoid was generated on the fly, both for test and train data.

The architecture was a 2-hidden-layer MLP with 40 hidden units per layer, transductive batch normalization (i.e. the batch statistics are not stored) and ReLU activations.
For MAML, the inner gradient descent had a step size of 0.01.
For iMAML, the inner gradient descent use line search for a maximum of 16 iterations.
The outer optimization was carried out for 70k steps using Adam~\citep{kingma2014adam} with a learning rate of $10^{-3}$ and all other hyperparameters set to their default values, the meta batch size was 25 and 10 samples were used for validation (outer) loss computation and for the inner loss definition.
The samples are generated on-the-fly.

\subsection{Compute}
All the experiments except the theorem illustrations were run on a public HPC cluster, providing NVIDIA V100 GPUs.
According to the numbers provided by this public HPC cluster, a maximum of 220 GPU days (5291 GPU hours) were used for the experiments (conservative upper bound), which includes all the reruns due to bugs, the experiments not reported in this paper and potential other projects worked on at the same time.

\section{Additional results}
\label{app:add-results}
\subsection{Inner iterations Overfitting for DEQs on training data}
While in \autoref{fig:deq-inner-opt-time-overf-all} the reported performance is the test set performance, the theory developed in \autoref{sec:theory} concerns only the training set performance.
In order to make sure that the behavior we are noticing on the test set performance is not due to some effect of lack of generalization we also report the same figure for ImageNet on training data in \autoref{fig:deq-inner-opt-time-overf-imagenet-train}.

\begin{figure}
    \centering
    \includegraphics{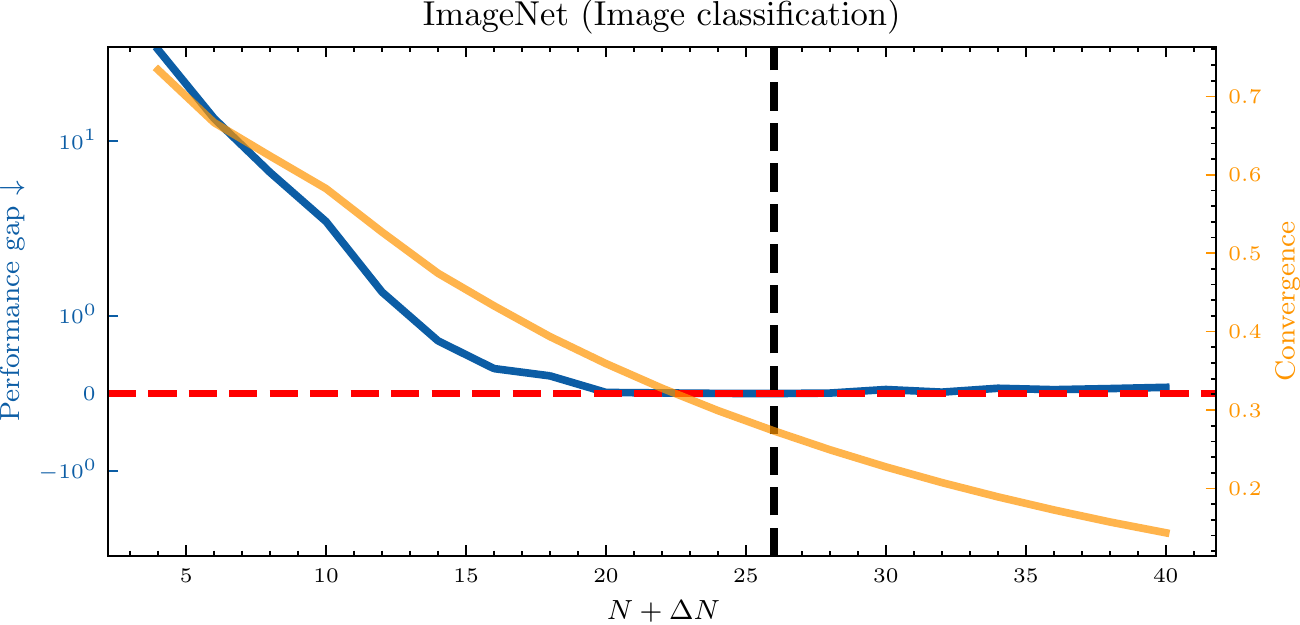}
    \caption{\textbf{Inner iterations overfitting for DEQs on train data.} We report the training set error for different inner optimization times.}
    \label{fig:deq-inner-opt-time-overf-imagenet-train}
\end{figure}

\subsection{Inner iterations Overfitting for DEQs on training loss}
While in \autoref{fig:deq-inner-opt-time-overf-all} the reported performance is not the training loss but for example for image classification the error, the theory developed in \autoref{sec:theory} concerns only the training loss.
In order to make sure that the behavior we are noticing is not due to a discrepancy between the optimized loss and the performance we also report the same figure for ImageNet on training loss in \autoref{fig:deq-inner-opt-time-overf-imagenet-loss}.

\begin{figure}
    \centering
    \includegraphics{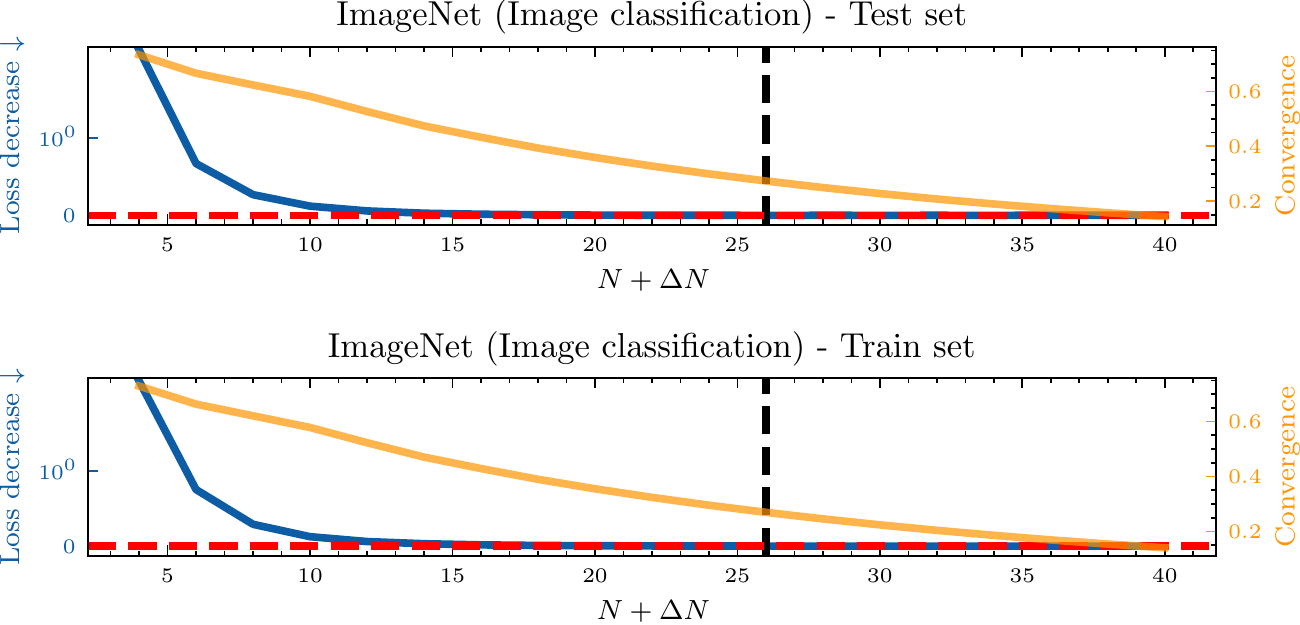}
    \caption{\textbf{Inner iterations overfitting for DEQs on the training loss.} We report the training set error for different inner optimization times.}
    \label{fig:deq-inner-opt-time-overf-imagenet-loss}
\end{figure}

\subsection{Lower bound for non overparametrized inner procedures but with non strongly convex inner and outer problems}
\begin{figure}
    \centering
    \includegraphics{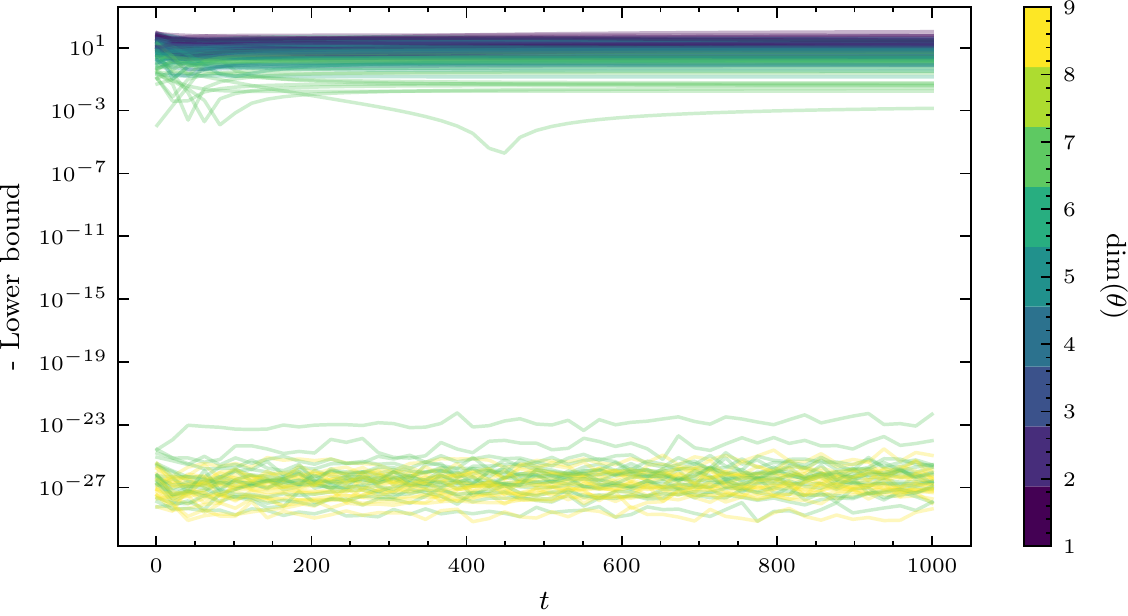}
    \caption{\textbf{Impact of inner problem underparametrization for non strongly convex cases.} Negative Lower bound, i.e. $\frac12 \| \left(\mathcal{P}(K_\text{out} K_\text{in}) - \mathcal{P}(K_\text{out} K_\text{in} E_N U )\right) (K_\text{out} r_N - z^\star)  \|_2^2$, from \autoref{thm:quadratic}. The inner and outer problem are not strongly convex, and the dimension of $z$ is 10. Roughly speaking, the inner problem would therefore be fully parameterized in $\theta$ if it was in dimension 10. We compute the lower bound for different inner optimization times and 20 seeds.}
    \label{fig:quad-U-not-surj-not-strongly-conv}
\end{figure}
We see in \autoref{fig:quad-U-not-surj-not-strongly-conv} that the lower bound can reach 0 for a much smaller $\theta$ dimension when the inner and outer problem are not strongly convex than when they are.

\subsection{Inner iterations Overfitting in Inverse Problems}
\citet{zou2023deep} introduced a new method termed ELDER, where the fixed-point defining function of DEQs is defined as the gradient of another function.
They compare this new approach with the one where the fixed-point defining function has a direct expression.

\begin{figure}
    \centering
    \includegraphics{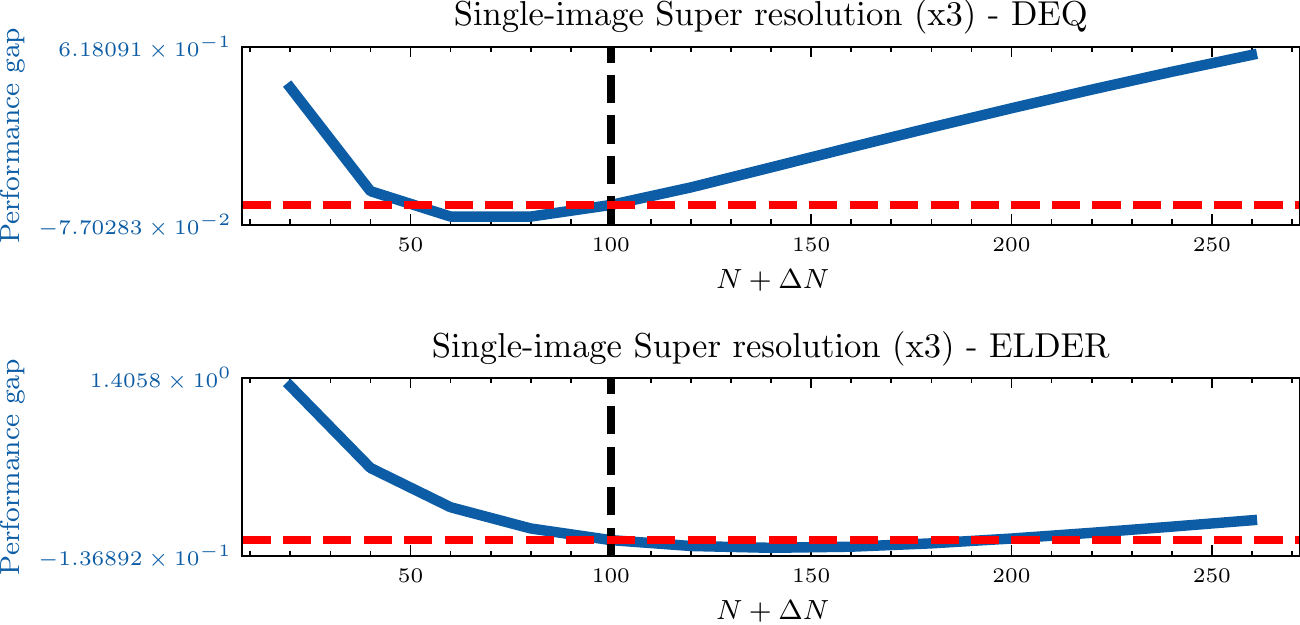}
    \caption{\textbf{Inner iterations Overfitting for Inverse Problems}: Difference between the test loss using $N+\Delta N$ iterations at inference and using $N$ iterations for models trained with $N$ iterations, $D(N, \Delta N) = L(\theta^{\star, N}, N) - L(\theta^{\star, N}, N + \Delta N)$. The black dashed line is at $\Delta N = 0$, i.e. the training number of inner iterations. The red dashed line is at 0, for easier visualization.
    }
    \label{fig:deq-inner-opt-time-overf-ip}
\end{figure}
As can be seen in \autoref{fig:deq-inner-opt-time-overf-ip}, the networks trained by \citet{zou2023deep} do not exhibit an Inner iterations Overfitting as strong as those we see in \autoref{fig:deq-inner-opt-time-overf-all}.
However, we clearly see that in one case it's better to use fewer iterations than at training time (for DEQ) and in the other case it's better to use more iterations (for ELDER).
A middle ground can clearly be achieved by picking the number of iterations used during training.
The reasons for these mismatches could be:
\begin{itemize}
    \item The outer optimization is not optimal.
    \item We do not have an overparametrized (or close to) inner solving procedure.
\end{itemize}
It is also quite surprising that the networks exhibit less convergence stability than the ones shown in \autoref{fig:deq-inner-opt-time-overf-all}.

\end{document}
